\documentclass{article}
\usepackage[a4paper,top=3cm,bottom=3cm,left=2.1cm,right=2.1cm,marginparwidth=1.75cm]{geometry}
\usepackage[utf8]{inputenc}
\usepackage{bm}
\usepackage{amsmath, amsfonts, amssymb, amsthm}
\usepackage{mathrsfs, dsfont}
\usepackage[dvipsnames]{xcolor}
\usepackage{graphicx}
\usepackage{float}
\usepackage{subfig}
\usepackage{bigints}
\def\twoplotswidth{0.47\linewidth}

\def\oneplotwidth{0.9\linewidth}

\usepackage[colorlinks=true,allcolors=blue]{hyperref}
\usepackage{xparse} % use of multiple default parameters for new command
\usepackage{hyperref}
\usepackage{multirow}

\usepackage{tasks}
\settasks{style=itemize}

\usepackage[round]{natbib}
\usepackage{pifont}

\newtheorem{theorem}{Theorem}[section]

\newtheorem{definition}[theorem]{Definition}

\newtheorem{proposition}[theorem]{Proposition}

\usepackage{algorithm}
\usepackage{algpseudocode}

\theoremstyle{remark}
\newtheorem{remark}{Remark}[theorem]
\newtheorem{example}{Example}[theorem]

\numberwithin{equation}{section}

\newenvironment{sqremark}{\begin{remark}}{\hfill \tiny $\blacksquare$ \end{remark}}
\newenvironment{sqexample}{\begin{example}}{\hfill \tiny $\blacksquare$ \end{example}}

\newcommand{\R}{\mathbb{R}}
\newcommand{\N}{\mathbb{N}}

\newcommand{\E}{\mathbb{E}}
\def\d{\mathrm{d}}
\def\P{\mathbb{P}}
\newcommand{\Q}{\mathbb{Q}}
\newcommand{\F}{\mathcal{F}}
\renewcommand{\Re}{\, \mathfrak{Re}}

\newcommand{\alphabet}[1][d]{{A}_{#1}}
\newcommand{\TA}[1][d]{T(\R^{#1})}
\newcommand{\eTA}[1][d]{T((\R^{#1}))}
\newcommand{\tTA}[2][d]{T^{#2}(\R^{#1})}

\newcommand{\word}[1]{{\mathcolor{NavyBlue}{\mathbf{#1}}}}
\newcommand{\emptyword}{{\color{NavyBlue}\textup{\textbf{\o{}}}}}
\newcommand{\proj}[1]{|_{\word{#1}}}

\newcommand{\conpow}[1]{^{\otimes #1}}

\newcommand{\shuprod}{\mathrel{\sqcup \mkern -3.2mu \sqcup}}
\newcommand{\shupow}[1]{^{\shuprod #1}}
\newcommand{\shuexp}[1]{e \shupow{#1}}

\newcommand{\LL}{\textnormal{LL}}

\NewDocumentCommand{\sig}{O{t} O{W}}{\widehat{\mathbb{#2}}_{#1}}
\NewDocumentCommand{\sigX}{O{t} O{X}}{\mathbb{#2}_{#1}}
\NewDocumentCommand{\sigE}{O{t} O{W}}{\E[\sig[#1][#2]]}

\newcommand{\bracket}[2]{\left \langle #1, #2 \right \rangle}
\NewDocumentCommand{\bracketsig}{O{t} O{W} m}{\bracket{#3}{\sig[#1][#2]}}
\NewDocumentCommand{\bracketsigX}{O{t} O{X} m}{\bracket{#3}{\sigX[#1][#2]}}
\NewDocumentCommand{\bracketsigE}{O{t} O{W} m}{\bracket{#3}{\sigE[#1][#2]}}
\NewDocumentCommand{\bracketsigtrunc}{O{M} O{t} O{W} m}{\bracket{#4}{\sig[#2][#3]^{\leq #1}}}

\newcommand{\balpha}{\bm{\alpha}}
\newcommand{\bsigma}{\bm{\sigma}}
\newcommand{\bxi}{\bm{\xi}}
\newcommand{\bpsi}{\bm{\psi}}

\newcommand{\bell}{\bm{\ell}}
\newcommand{\bp}{\bm{p}}

\usepackage{mathtools}
\usepackage{authblk}
\mathtoolsset{showonlyrefs}
\newcommand{\heston}{Heston, $T=1/2, \theta=0.0625, \kappa=4, \eta=0.7$.}
\newcommand{\fbergomi}{Shifted fractional Bergomi, $T=1/12, f_0(t) \equiv 0.0625, H=0.1, \epsilon=1/52, \varsigma=1.6$.}
\newcommand{\fbergomisix}{Shifted fractional Bergomi, $T=1/2, \xi_0=0.0625, H=0.1, \epsilon=1/52, \varsigma=1.6$.}
\newcommand{\eurocall}{European call option with at-the-money strike}
\newcommand{\asiancall}{Geometric Asian call option with at-the-money fixed strike}
\newcommand{\lookbackcall}{Look-back call option with floating strike}
\newcommand{\BASEcaption}{the Fourier solution (a) (blue), Fourier signature approximation (b) (blue), Vanilla NN (orange), Signature NN (green) and Recurrent NN (red)}
\newcommand{\PDFcaption}{Out-of-sample P\&L density of the delta hedging associated with \BASEcaption.}
\newcommand{\MSPcaption}{Out-of-sample mean and mean squared P\&L of the delta hedging associated with \BASEcaption.}
\newcommand{\TRUNCcaption}{Out-of-sample mean squared P\&L of the delta hedging associated with the the Fourier solution (a) (blue), Fourier signature approximation (b) (blue), Vanilla NN (orange), Signature NN (green), Recurrent NN (red) and Log-Signature NN (purple), for several truncation orders.}
\newcommand{\TRAJcaption}{Two out-of-sample spot trajectories (top) and the delta hedging (bottom) associated with \BASEcaption.}
\newcommand{\TIMEcaption}{Training time for the Vanilla NN, Signature NN and Recurrent NN.}
\newcommand{\CONVcaption}{Out-of-sample mean squared P\&L as a function of training time (seconds) for the Signature NN for different truncation orders.}
\newcommand{\CONVcaptionLog}{Out-of-sample mean squared P\&L as a function of training time (seconds) for the Log-Signature NN for different truncation orders.}

\usepackage{enumitem}
\usepackage[normalem]{ulem}

\usepackage{tikz}
\usetikzlibrary{positioning, fit, arrows.meta, shapes, math}

\title{Hedging with memory: shallow and deep learning with signatures}
\author[1]{Eduardo Abi Jaber \thanks{eduardo.abi-jaber@polytechnique.edu. The first author is grateful for the financial support from the Chaires FiME-FDD and Financial Risks at École Polytechnique.}}
\author[2,3]{Louis-Amand Gérard \thanks{louis-amand.gerard@etu.univ-paris1.fr. \\ We would like to thank Olivier Guéant for fruitful discussions.}}
\affil[1]{École Polytechnique, CMAP}
\affil[2]{Université Paris 1 Panthéon-Sorbonne, CES}
\affil[3]{Gefip}

\begin{document}

\maketitle

\begin{abstract}
    We investigate the use of path signatures in a machine learning context for hedging exotic derivatives under non-Markovian stochastic volatility models. In a deep learning setting, we use signatures as features in feedforward neural networks and show that they outperform LSTMs in most cases, with orders of magnitude less training compute. In a shallow learning setting, we compare two regression approaches: the first directly learns the hedging strategy from the expected signature of the price process; the second models the dynamics of volatility using a signature volatility model, calibrated on the expected signature of the volatility. Solving the hedging problem in the calibrated signature volatility model yields more accurate and stable results across different payoffs and volatility dynamics.
\end{abstract}

\begin{description}
    \item[Mathematics Subject Classification (2010):] 60L10, 91G20, 91G60 
    \item[Keywords:] Deep-hedging, non-Markovian stochastic volatility models, path-signatures,  exotic derivatives, Fourier methods.
\end{description}

\section{Introduction}
    
    The problem of hedging derivatives represents a central challenge in financial markets. Under Markovian models, the theory is very well developed, specifically for European-style derivatives. However, significant challenges arise when considering path-dependent options where the payoff depends on the asset's entire price path, or further still, when the underlying asset has non-Markovian dynamics, where conventional parametrized hedging approaches tend to be too restrictive or untractable. In response to these challenges, non-parametric approaches have gained a lot of popularity, and more specifically with the improvement of machine learning software and hardware, deep hedging approaches for their versatility, ease of train and ability to capture nonlinearities, see for instance \citet{deephedgingteichmann}. However, standard neural networks present certain trade-offs. Feedforward networks for example can not retain memory and are thus not well-suited to process historical data by themselves. Moreover, LSTMs (Long Short Term Memory) have been designed in \citet{LSTM} to recursively handle together new information and its stored memory, but at the cost of poor parallelization and significantly slower training.

    \vspace{0.5em}
    \paragraph{Aim.}
    The first aim of this paper is to investigate the use of (truncated) path signatures {in a deep learning context} as a feature engineering technique to alleviate both challenges. The signature of a path provides a universal representation basis for its whole trajectory allowing the processing of memory in a feedforward network, without the limitations of recursivity and slow training. For this purpose, we will compare a \textit{simple} feedforward neural network enhanced with signature features, the SNN, with a LSTM network of the same size with only Markovian features, the RNN. We will also compare them in terms of computational efficiency and how long it takes to reach certain performance marks. As we will show, signature features capture most of the path-dependent information required for hedging while allowing for a very fast training.\\
    
    The second aim of this paper is to compare two data-driven signature-based methods in a shallow learning context. The first, non-parametric, uses data to recover a hedging strategy that is linear in the signature of the (time-augmented) price process, following the approach introduced in \cite{arribas-nonparam}. The second uses data to recover the volatility of a the price process as a linear combination of the signature of an underlying (time-augmented) Brownian motion, enabling semi-analytic computations via Fourier inversion, building on \cite{sigvolfourier}. We will study those approaches for both European and path-dependent derivatives and under both Markovian and non-Markovian models. We will see that for Fourier-invertible payoffs, calibrating the signature volatility significantly improves the hedging compared to using the completely model-free hedging method.

    \vspace{0.5em}
    \paragraph{Related literature.}
    Path signatures have been applied to finance for over a decade now, specifically in the context of machine learning as feature set. As a first notable mention, \citet{linregclassif} have used signatures as linear classifiers for atypical market behaviors in a multidimensional setting. Later, \citet{uat_2013} have proposed the now famous universal approximation theorem (UAT) for linear forms on the signature which they applied to recover an approximation of a process' dynamics as a function of the signature of a Brownian motion. Building upon it, \citet{arribas-nonparam} have developed a non-parametric method for pricing options by regressing their payoffs' dynamics against the signature of the observable spot price, and hedging by then doing a non-convex optimization. Their method has the main advantages of being model-free, data-driven and working with most exotic payoffs, but it is still subject to the curse of dimensionality. To alleviate some of these limitations, \citet{sigvolfourier} have derived Fourier inversion methods for a certain signature stochastic volatility model, in order to leverage known and fast algorithms while remaining very general. As last notable mentions, \citet{sig_trading_futter} have used path signatures for mean-variance portfolio optimization and \citet{cuchiero2023spvix} have extended the Sig-SDEs framework of \citet{sig_sde} to jointly calibrating SPX and VIX options.
    
    Furthermore, path signatures have also been applied more specifically in the context of deep learning, at first on character and movement recognition \cite{sig_handwritten, sig_movement}, as the signature is invariant under time reparametrization and then more broadly, for example in \cite{sig_moods1, sig_moods2} for emotions and behavioral classifications or in \citet{bayer_sigcontrols, bayer2023optimal, bayeramericanoptions} for optimal control, optimal stopping and option pricing respectively.

    Finally, deep hedging has been at the center of recent research. For instance \citet{deephedgingteichmann}, later extended to rough volatility in \citet{deephedging_rough}, present a framework for hedging a portfolio of derivatives in the presence of market frictions, \citet{fromstochroughvoldeephedging} have proposed to apply a GRU network \citet{gru} for exponential hedging of European call options in the under rough Bergomi model. In a similar framework of hedging a European options under rough Bergomi, \citet{SigFormer} have used path signatures together with a transformer \cite{transformer}. \\

    We refer the interested reader to \citet{guidedtourquadhedging} for a comprehensive tour of quadratic hedging approaches, \citet{nn_opt_price_review} for a literature review on deep option pricing, \citet{chevyrev2016primer} for an introduction to path signatures and \cite{kidger2019deep} for an introduction the use of signature transforms as layers inside a neural network.

    \vspace{0.5em}
    \paragraph{Outline.}
    In sections~\ref{subsec:tensor} and \ref{subsec:pathsig} we will briefly introduce the framework of signatures and some of their properties. Section~\ref{sec:quad_hedging} we will recall the quadratic hedging problem and give a summaries of the three methods that we will compare in Sections~\ref{sec:sig_nn} and \ref{sec:sig_reg}. Finally in Appendix~\ref{app:algorithms} we explicit the algorithms used, in Appendix~\ref{app:sig_order_sup2} we discuss some artifacts of our training and in Appendix~\ref{app:arribas-timings} we detail some computing times.

\section{Reminder on signature}

\subsection{Tensor algebra} \label{subsec:tensor}
    
    Let $d \in \N$ and denote by $\otimes$ the tensor product over $\R^d$, e.g.~$(x \otimes y \otimes z)_{ijk} = x_i y_j z_k$, for $i, j, k = 1, \dots, d$, for $x, y, z \in \R^d$. For $n\geq 1$, we denote by $(\R^d) \conpow{n}$ the space of tensors of order $n$ and by $(\R^d) \conpow{0} = \R$. In the sequel, we will consider mathematical objects, path signatures, that live on the extended tensor algebra space $\eTA $ over $\R^d$, that is the space of (infinite) sequences of tensors defined by
    $$ \eTA := \left\{ \bell = (\bell^n)_{n=0}^\infty : \bell^n \in (\R^d) \conpow{n} \right\}. $$
    
    Similarly, for $M \geq 0$, we define the truncated tensor algebra $\tTA{M}$ as the space of sequences of tensors of order at most $M$ defined by
    $$ \tTA{M} := \left\{ \bell \in \eTA : \bell^n = 0, \text{ for all } n > M \right\}, $$
    
    and the tensor algebra $\TA$ as the space of all finite sequences of tensors defined by
    $$ \TA := \bigcup_{M \in \N } \tTA{M}. $$
    
    We clearly have $\TA \subset \eTA$.
    For $\bell = (\bell^n)_{n \in \N}, \bp = (\bp^n)_{n \in \N} \in \eTA$ and $\lambda \in \R$, we define the following operations:
    \begin{align*}
        \bell + \bp :&
        = (\bell^n + \bp^n)_{n \in \N}, \quad \bell \otimes \bp :
        = \left( \sum_{k=0}^n \bell^k \otimes \bp^{n-k} \right)_{n \in \N}, \quad 
        \lambda \bell :
        = (\lambda \bell^n)_{n \in \N}.
    \end{align*}
    
    These operations induce analogous operations on $\tTA{M}$ and $\TA$. \\
    
    Let $\{ e_1, \dots, e_d \} \subset \R^d$ be the canonical basis of $\mathbb{R}^d$ and $\alphabet = \{ \word{1}, \word{2}, \dots, \word{d} \}$ be the corresponding alphabet. To ease reading, for $i \in \{ 1, \dots, d \}$, we write $e_{i}$ as the blue letter $\word{i}$ and for $n \geq 1, i_1, \dots, i_n \in \{ 1, \dots, d \}$, we write $e_{i_1} \otimes \cdots \otimes e_{i_n} $ as the concatenation of letters $\word{i_1 \cdots i_n}$, that we call a word of length $n$. We note that $(e_{i_1} \otimes \cdots \otimes e_{i_n})_{(i_1, \dots, i_n) \in \{ 1, \dots, d \}^n}$ is a basis of $(\R^d) \conpow{n}$ that can be identified with the set of words of length $n$ defined by 
    \begin{equation} \label{eq:sig_basis}
        V_n := \{ \word{i_1 \cdots i_n}: \word{i_k} \in \alphabet \text{ for } k = 1, 2, \dots, n \}.
    \end{equation}
    
    Moreover, we denote by $\emptyword$ the empty word and by $V_0 = \{ \emptyword \}$ which serves as a basis for $(\R^d) \conpow{0} = \R$. It follows that $V := \cup_{n \geq 0} V_n$ represents the topological basis of $\eTA$. In particular, every $\bell \in \eTA$ can be decomposed as
    \begin{equation} \label{eq:sig_expansion}
        \bell = \sum_{n=0}^\infty \sum_{\word{v} \in V_n} \bell^{\word{v}} \word{v},
    \end{equation}
    
    where $\bell^\word{v}$ is the real coefficient of $\bell$ at coordinate $\word{v}$. We stress again that in the sequel, every blue `word` $\word{v} \in V$ represents an element of the canonical basis of $\eTA$, i.e.~there exists $n \geq 0$ such that $\word{v}$ is of the form $\word{v} = \word{i_1 \cdots i_n}$, which represents the element $e_{i_1} \otimes \cdots \otimes e_{i_n} $. The concatenation $\bell \word{v}$ of elements $\bell \in \eTA$ and the word $\word{v} = \word{i_1 \cdots i_n}$ means $\bell \otimes e_{i_1} \otimes \cdots \otimes e_{i_n}$. \\

    In addition to the decomposition \eqref{eq:sig_expansion} of elements $\bell \in \eTA$, we introduce the projection $\bell \proj{u} \in \eTA$ as
    \begin{equation} \label{eq:projection}
        \bell \proj{u} := \sum_{n=0}^\infty \sum_{\word{v} \in V_n} \bell^\word{vu} \word{v}
    \end{equation}
    for all $\word{u} \in V$. The projection plays an important role in the space of iterated integrals as it is closely linked to partial differentiation, in contrast with the concatenation that relates to integration.
    
    \begin{sqexample}
        Take the alphabet $\alphabet[3] = \{ \word{1}, \word{2}, \word{3} \}$ and let $\bell = 4 \cdot \emptyword + 3 \cdot \word{1} - 1 \cdot \word{12} + 2 \cdot \word{2212}$, then
        $$ \bell^{\emptyword} = 4 \cdot \emptyword, \quad \bell \proj{1} = 3 \cdot \emptyword, \quad \bell \proj{2} = - 1 \cdot \word{1} + 2 \cdot \word{221}, \quad \bell \proj{3} = 0. $$
    \end{sqexample}
    
    We now define the bracket between $\bell \in \TA$ and $\mathbb{X} \in \eTA$ by
    \begin{align} \label{eq:bracket}
        \langle \bell, \mathbb{X} \rangle 
        := \sum_{n=0}^{\infty} \sum_{\word{v} \in V_n} \bell^\word{v} \mathbb{X}^\word{v}. 
    \end{align}
    
    Notice that it is actually well defined as $\bell$ has finitely many non-zero terms. For $\bell \in \eTA$, the series in \eqref{eq:bracket} involves infinitely many terms and requires special care, which is discussed in \citep*{linearfbm} and \citep*{christa_sig_affine_polynomial}.
    
    \begin{definition}[Shuffle product] \label{def:shuffleprod}
        The shuffle product $\shuprod: V \times V \to \TA$ is defined inductively for all words $\word{v}$ and $\word{w}$ and all letters $\word{i}$ and $\word{j}$ in $\alphabet$ by
        \begin{align*}
            \word{v} \word{i} \shuprod \word{w} \word{j} &
            = (\word{v} \shuprod \word{w} \word{j}) \word{i} + (\word{v} \word{i} \shuprod \word{w}) \word{j},
            \\ \word{w} \shuprod \emptyword &
            = \emptyword \shuprod \word{w} = \word{w}.
        \end{align*}
    \end{definition}
    
    \begin{definition}[Half shuffle product] \label{def:halfshuffleprod}
        Similarly, the (right) half shuffle product $\succ: V \times V \to \TA$ is defined inductively for all words $\word{v}$ and $\word{w}$ and all letters $\word{i}$ and $\word{j}$ in $\alphabet$ by
        \begin{align*}
            \word{v} \word{i} \succ \word{w} \word{j} &
            = (\word{v} \word{i} \shuprod \word{w}) \word{j},
            \\ \word{w} \succ \emptyword &
            = 0.
        \end{align*}
    \end{definition}
        
    With some abuse of notation, the shuffle (resp. half shuffle) product on $\eTA$ induced by the shuffle (resp. half shuffle) product on $V$ will also be denoted by $\shuprod$, (reps. $\succ$). The shuffle product is clearly commutative. See \cite{reeshuffles} and \cite{gainesshuffle} for more information on the shuffle product.

\subsection{Path signature} \label{subsec:pathsig}

    We define the (path) signature of a semimartingale process as the sequence of iterated stochastic integrals in the sense of Stratonovich. Throughout the paper, the Itô integral is denoted by $\int_0^\cdot Y_t \d X_t$ and the Stratonovich integral by $\int_0^\cdot Y_t \circ \d X_t.$ If both $X$ and $Y$ are semimartingales then, we have the relation $\int_0^\cdot Y_t \circ \d X_t = \int_0^\cdot Y_t \d X_t + \frac{1}{2} [X,Y]_\cdot $.

    \begin{definition}[Signature] \label{def:sig}
        Fix $T > 0$. Let $(X_t)_{t \geq 0}$ be a continuous semimartingale in $\R^d$ on some filtered probability space $(\Omega, \F, (\F_t)_{t \geq 0}, \P)$. The signature of $X$ is defined by
        \begin{align*}
            \mathbb{X}: \Omega \times [0, T] &
            \to \eTA
            \\ (\omega, t) &
            \mapsto \sigX (\omega) := (1, \sigX^1(\omega), \dots, \sigX^n(\omega), \dots),
        \end{align*}
        where
        $$ \sigX^n := \int_{0 < u_1 < \cdots < u_n < t} \circ \d X_{u_1} \otimes \cdots \otimes \circ \d X_{u_n} $$
        takes value in $(\R^d)^{\otimes n}$, $n \geq 0$. Similarly, the truncated signature of order $M \in \N$ is defined by
        \begin{align} \label{def:sig-trunc}
            \mathbb{X}^{\leq M}: [0, T] &
            \to \tTA{M}
            \\ (\omega, t) &
            \mapsto \sigX^{\leq M}(\omega) := (1, \sigX^1(\omega), \dots, \sigX^M(\omega), 0, \dots, 0, \dots).
        \end{align}
    \end{definition}

    \begin{sqexample}
        Let $X = (X^1, X^2)$, then the first few orders of $\sigX$ are given by
        \begin{equation}
            \sigX^0 = 1,
            \quad
            \sigX^1 =
            \begin{pmatrix}
                X_t^1 \\
                X_t^2
            \end{pmatrix},
            \quad
            \sigX^2 =
            \begin{pmatrix}
                \frac{(X_t^1)^2}{2!} & \int_0^t X_s^1 \circ \d X_s^2 \\
                \int_0^t X_s^2 \circ \d X_s^1 & \frac{(X_t^2)^2}{2!}
            \end{pmatrix}.
        \end{equation}
    \end{sqexample}

    One of the interesting properties of signature is its ability to linearize products of the signature, just like the Cauchy product for power series.

    \begin{proposition}[Shuffle properties] \label{prop:shufflepropertyextended}
        For $N, M \in \N$, if $\bell_1 \in \tTA{M}, \bell_2 \in \tTA{N}$, then $\bell_1 \shuprod \bell_2 \in \tTA{M+N}$ and
        $$ \bracketsig{\bell_1} \bracketsig{\bell_2} = \bracketsig{\bell_1 \shuprod \bell_2}, $$
        and $\bell_1 \succ \bell_2 \in \tTA{M+N}$ and
        $$ \int_0^t \bracketsig[s]{\bell_1} \circ \d \bracketsig[s]{\bell_2} = \bracketsig{\bell_1 \succ \bell_2}. $$
    \end{proposition}
    
    \begin{proof}
        Recursively apply an integration by parts. See \cite{gainesshuffle} for more details.
    \end{proof}

\section{Quadratic hedging} \label{sec:quad_hedging}
    
    Let $(\Omega, \F, \Q)$ be a probability space supporting a two-dimensional Brownian motion $(W, W^{\perp})$ generating the filtration $(\F_t)_{t \geq 0}$. We set $B = \rho W + \sqrt{1 - \rho^2} W^{\perp}$
    for some $\rho \in [-1, 1]$. We consider that the dynamics of the risky asset $S$, under the risk neutral probability measure $\Q$, are given by a stochastic volatility model where the volatility process $\Sigma$ is an adapted process such that $\int_0^T \E [\Sigma_t^2] \d t < \infty$ and
    \begin{align}
        \d S_t &= S_t \Sigma_t \d B_t. \label{eq:dynaprice}
    \end{align}
    
    Let $\xi$ be an $\F_T$-measurable non-negative random variable such that $\E[\xi^2] < \infty$ that we are looking to hedge using a self-financing portfolio. A self-financing hedging portfolio $X$ consists of an initial wealth $X_0 \in \R$ and a progressively measurable strategy $(\alpha_t)_{t \leq T}$ of the amount of shares invested in asset $S$ at time $t \leq T$, which leads to the following dynamics 
    \begin{align}
        X_t^\alpha = X_0 + \int_0^t \alpha_u \d S_u,
    \end{align}
    with $\alpha$ in the set of admissible hedging strategies $\mathcal{A}$ defined by
    \begin{align}
        \mathcal{A} = \left \{ \alpha \text{ progressively measurable such that } \int_0^T \E \left[ ( \alpha_t S_t \Sigma_t )^2 \right] \d t < \infty \right\}.
    \end{align}

    A quadratic hedging strategy aims at minimizing the following objective function
    \begin{align} \label{eq:Jquad}
        P(X_0, \alpha) = \E \left[ \left( X_T^\alpha - \xi \right)^2 \right] 
    \end{align}
    over $X_0 \in \R$ and $\alpha \in \mathcal{A}$. Following \citet{hedging_foellmer} and \citet{guidedtourquadhedging}, it is easy to derive a solution of the quadratic hedging problem using the martingale representation theorem. Note that $(\E[\xi | \F_t])_{t \leq T}$ is a square integrable martingale with terminal value $\xi$ at $T$. An application of the martingale representation theorem \cite[Theorem 4.15]{martingalerep} yields the existence of two progressively measurable and square integrable processes $Z$ and $Z^{\perp}$ such that 
    \begin{align} \label{eq:martingalerep}
        \E \left[\xi | \F_t \right ] = \xi - \int_t^T Z_s \d W_s - \int_t^T Z^{\perp}_s \d W_s^{\perp}.
    \end{align}
    The optimum is attained for $(X_0^*, \alpha^*)$ given by
    \begin{align} \label{eq:hedgingoptimal}
        X_0^* = \E \left[ \xi \right] \quad \text{and} \quad \alpha_t^* = \frac{1}{S_t \Sigma_t} \left( \rho Z_t + \sqrt{1 - \rho^2} Z_t^\perp \right), \quad t \leq T.
    \end{align}
    
    We compare three numerical approaches for solving this quadratic hedging problem:
    \begin{itemize}
        \item[\textbullet] \textbf{Fourier-based method in signature volatility models} (Section~\ref{subsec:sighedging}): An exact method using Fourier inversion, applicable when the payoff $\xi$ belongs to a suitable class and the volatility process $\Sigma$ is modeled as a linear combination of truncated signatures of a Brownian motion. This approach yields semi-explicit solutions for \eqref{eq:hedgingoptimal} and will serve as a benchmark.
        
        \item[\textbullet] \textbf{Deep hedging with and without signatures as features} (Section~\ref{subsec:deephedging}): A data-driven method that parameterizes the self-financing portfolio using neural networks with various architectures and input features.
        
        \item[\textbullet] \textbf{Shallow linear payoff/strategy approximation} (Section~\ref{subsec:arribas}): An approach where both the payoff and the hedging strategy is approximated by a linear combination of signature terms.
    \end{itemize}
    
    \begin{remark}
        Although alternative objective criterions could be considered, we focus here on the quadratic criterion due to its analytical tractability. It  provides a framework for evaluating and comparing approximation methods, such as deep learning or signature-based approaches, against the semi-explicit benchmark provided by the Fourier method, whenever the latter is applicable.
    \end{remark}

\subsection{Fourier hedging under signature volatility models} \label{subsec:sighedging}

    We start by summarizing the signature volatility hedging method of \cite{sigvolfourier}. It will be used as benchmark in the examples of Section~\ref{sec:sig_nn} and as main point of comparison in Section~\ref{sec:sig_reg}. \\
    
    It is valid for models where the stochastic volatility $\Sigma$ in \eqref{eq:dynaprice} is given as a linear combination of the truncated signature of the time augmented process $\widehat{W}_t:=(t, W_t)$, i.e.
    \begin{align}
        \Sigma_t &= \bracketsig{\bsigma_t}, \label{eq:sigvol}
    \end{align}
    where $\bsigma: [0, T] \to \tTA[2]{M}$, $M \geq 0$. In this setting, the following expression holds for the characteristic functional of $\log S$:
    \begin{equation} \label{eq:charfun}
         \E \left[ \left. \exp \left( \int_t^T f(s, u) \d \log S_s \right) \right| \mathcal{F}_t \right]
         = \exp \left( \bracketsig{\bpsi_t(u)} \right) =: \phi_t(u), \quad t \leq T,
    \end{equation}
    where $\bpsi$ solves the following Riccati
    \begin{align} \label{eq:Ric}
        -\dot{\bpsi}_t(u) &
        = \frac{1}{2} (\bpsi_t(u) \proj{2}) \shupow{2} + \rho f(t, u) (\bsigma_t \shuprod \bpsi_t(u) \proj{2}) + \frac{1}{2} \bpsi_t(u) \proj{22} + \bpsi_t(u) \proj{1} + \frac{f(t, u)^2 - f(t, u)}{2} \bsigma_t \shupow{2},
        \\ \bpsi_T(u) &
        = 0.
    \end{align}
    
    Solving $\bpsi$, and hence $\phi$, numerically opens the door to computing the optimal initial wealth and hedging strategy in~\eqref{eq:hedgingoptimal} using Fourier inversion techniques, for specific payoffs $\xi$. For instance, the optimal initial wealth $X_0^*$ can be readily recovered with
    $$ X_0^* = X_0^\textnormal{BS} - \frac{K}{\pi} \int_0^\infty \Re \left[ e^{-i (u - \frac{i}{2}) \log K} \left( M_0 \left( u - \tfrac{i}{2} \right) - M_0^\textnormal{BS} \left( u - \tfrac{i}{2} \right) \right) \right] \frac{\d u}{\left( u^2 + \tfrac{1}{4} \right)}, $$
    together with the optimal hedging strategy $\alpha^*$ in \eqref{eq:hedgingoptimal}
    $$ \alpha_t^* = \alpha_t^\textnormal{BS} - \frac{K}{\pi} \int_0^\infty \Re \left[ e^{-i (u - \frac{i}{2}) \log K} \zeta_t(u - \tfrac{i}{2}) \right] \frac{\d u}{\left( u^2 + \tfrac{1}{4} \right)}, $$
    where 
    $$ \zeta_t(u) := \frac{f(t, u)}{S_t} \left( M_t(u) - M_t^\textnormal{BS}(u) \right) + \frac{\rho}{S_t \Sigma_t} M_t(u) \bracketsig{\bpsi_t(u) \proj{2}}, $$
    and $M_t(u) := \phi_t(u) \exp(\int_0^t f(s, u) \d \log S_s)$ with $f(t, u)$ specific to each payoff, e.g.
    % and $M_t(u) := \phi_t(u) e^{U_t(u)}$ with $U_t(u)$ and $f(t, u)$ specific to each payoff, e.g.
    \begin{itemize}
        \item if $\xi = (S_T - K)^+$ is a European call with strike $K$ and maturity $T$, then $f(t, u) = iu$ %and $U_t(u) = iu \log S_t$,
        \item if $\xi = (\exp[\frac{1}{T} \int_0^T \log S_t \d t] - K)^+$ is a geometric Asian call, then $f(t, u) = iu \frac{T-t}{T}$ %and
        % $$ U_t(u) = \int_0^t f(s, u) \d \log S_s + f(t, u) \log S_t. $$
    \end{itemize}
    Moreover, in order to improve the speed of convergence of the Gauss-Laguerre quadrature towards the Fourier integral, Black-Scholes (BS) control variate is included with $M_t^\textnormal{BS}(u) := \phi_t^\textnormal{BS}(u) e^{U_t(u)}$, $X_0^\textnormal{BS}$ and $\alpha^\textnormal{BS}$ the characteristic functional, initial wealth and delta hedging under Black-Scholes model.

\subsection{Deep hedging} \label{subsec:deephedging}

    Deep hedging as studied by \citep*{deephedgingteichmann} parametrizes the hedging strategy $\alpha = \alpha^{\theta}$ and the initial wealth $X_0 = X_0^{\theta}$ by a neural network with parameters $\theta$, and then trains the network to find the $\theta$ that minimizes the objective criterion~\eqref{eq:Jquad}. \\
    
    In practice, we discretize time as $0 = t_0 < \dots < t_J = T$, choose input features $X_{t_j}$ for $j = 0, \dots, J-1$, e.g., $X_{t_j} = (t_j, S_{t_j})$ for a naive feedforward network or $X_{t_j} = (t_j, S_{t_j}, \alpha^\theta(X_{t_{j-1}}))$ for a recurrent one, and fix a class of architectures $\Theta$. The training problem then becomes:
    \begin{align}
        \min_{\theta \in \Theta} \frac{1}{I} \sum_{i=1}^I \left( X_0^\theta + \sum_{j=0}^{J-1} \alpha^\theta (X_{t_j}^i) \Delta S_{t_{j+1}}^i - \xi^i \right)^2,
    \end{align}
    where $\Delta S_{t_{j+1}} = S_{t_{j+1}} - S_{t_j}$ and $I$ is the number of training samples. \\
    
    Thanks to the universal approximation property of neural networks, restricting the search to a parametrized class is not too limiting: in principle, neural networks can approximate the optimal solution arbitrarily well.
    
    However, the choice of architecture $\Theta$ and input features $X$ is crucial and can drastically affect performance. More complex architectures, such as recurrent networks, often perform better on sequential tasks, as evidenced in fields like language modeling~\citep{google_lstm, transformer} and reinforcement learning~\citep{deepmind_lstm, openai_lstm}. Yet these models also increase computational costs and may not be parallelizable, as is the case with LSTM~\citep{LSTM} or GRU~\citep{gru} networks.\\
    
    The goal of Section~\ref{sec:sig_nn} is to evaluate, within the deep hedging framework, whether signature features can efficiently capture relevant path information and mitigate the need for recurrent complexity. Specifically, we will compare a vanilla feedforward network using signature features to a recurrent LSTM using Markovian inputs. Figure~\ref{fig:neurons} highlights the difference in architectures and complexity.

\def\xX{-1.25}
\def\xXi{-.25}
\def\xstart{0}
\def\xplus{1}
\def\xgates{2.5}
\def\xforget{4}
\def\xinput{5}
\def\xoutput{6}
\def\xfinish{7}
\def\xY{8.25}
\def\yshortup{0}
\def\yshortmid{-0.5}
\def\yshort{-1}
\def\yforget{-2}
\def\yupdate{-3}
\def\yinput{-4}
\def\youtput{-5}
\def\ylongmid{-5.5}
\def\ylongdown{-6}
\def\rounded{6pt}

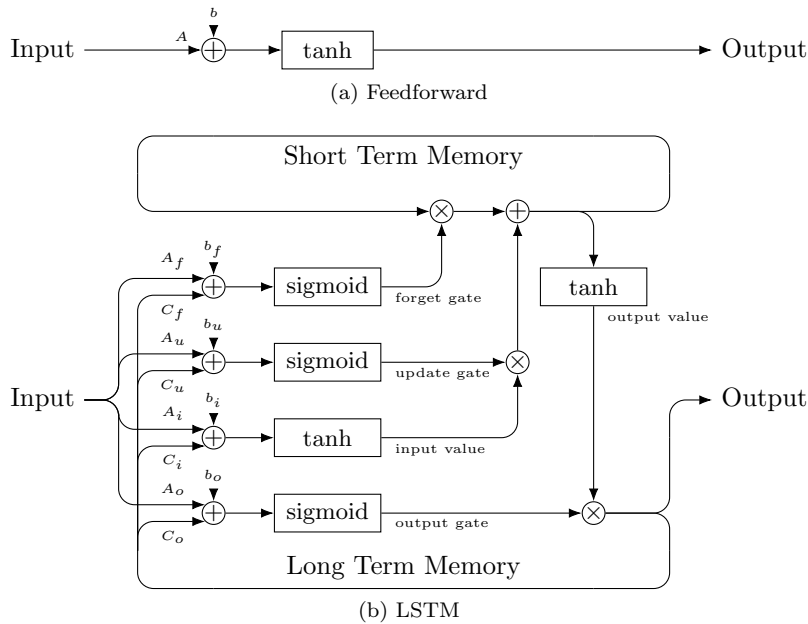
\begin{figure}[H]
    \centering
    \subfloat[Feedforward]{
    \begin{tikzpicture}[
        font=\tiny,
        >=LaTeX,
        operator/.style={
            font=\normalsize,
            circle,
            draw,
            inner sep=-0.5pt,
            minimum height =.3cm,
            },
        data/.style={
            font=\normalsize,
            },
        activation/.style={
            font=\small,
            rectangle,
            draw,
            minimum width=12mm,
            minimum height=5mm,
            inner sep=2pt
            },
        ]
    
        \node[data] (X) at (\xX,0) {Input};
        \node [operator] (add) at (\xplus,0) {$+$};
        \node (addb) at (\xplus,0.5) {$b$};
        \node [activation] (act) at (\xgates,0) {$\tanh$};
        \node[data] (Y) at (\xY,0) {Output};
        
        \draw [->] (addb) -- (add.north);
        \draw [->] (X) -- (add) node[xshift=-12pt,above] () {$A$};
        \draw [->] (add) -- (act);
        \draw [->] (act) -- (Y);
    \end{tikzpicture}
    }
    \\
    \subfloat[LSTM]{
    \begin{tikzpicture}[
        font=\tiny\raggedleft,
        >=LaTeX,
        operator/.style={
            font=\small,
            circle,
            draw,
            inner sep=-0.5pt,
            minimum height =.3cm,
            },
        data/.style={
            font=\normalsize,
            },
        activation/.style={
            font=\small,
            rectangle,
            draw,
            minimum width=14mm,
            minimum height=5mm,
            inner sep=2pt
            },
        label/.style = {
            inner sep=2pt,
            below right=0pt and 3pt of #1
            },
        ]
    
        \node[data] (X) at (\xX,\yupdate/2+\yinput/2) {Input};
        \node[data] (Y) at (\xY,\yupdate/2+\yinput/2) {Output};
        \coordinate (Ym) at (\xfinish,\yinput) {};
        \coordinate (Xi) at (\xXi,\yupdate) {};
        \coordinate (Him) at (\xstart,\ylongmid) {};
        \coordinate (Hom) at (\xfinish,\ylongmid) {};
        \coordinate (Hi) at (\xplus,\ylongdown) {};
        \coordinate (Ho) at (\xoutput,\ylongdown) {};
        \coordinate (Cim) at (\xstart,\yshortmid) {};
        \coordinate (Com) at (\xfinish,\yshortmid) {};
        \coordinate (Ci) at (\xplus,\yshortup) {};
        \coordinate (Co) at (\xoutput,\yshortup) {};
        \node (faddb) at (\xplus,\yforget+0.5) {$b_f$};
        \node (uaddb) at (\xplus,\yupdate+0.5) {$b_u$};
        \node (iaddb) at (\xplus,\yinput+0.5) {$b_i$};
        \node (oaddb) at (\xplus,\youtput+0.5) {$b_o$};
        
        \node [activation] (fact) at (\xgates,\yforget) {\textnormal{sigmoid}};
        \node [activation] (uact) at (\xgates,\yupdate) {\textnormal{sigmoid}};
        \node [activation] (iact) at (\xgates,\yinput) {$\tanh$};
        \node [activation] (oact) at (\xgates,\youtput) {\textnormal{sigmoid}};
        \node [activation] (sact) at (\xoutput,\yforget) {$\tanh$};
    
        \node [operator] (fadd) at (\xplus,\yforget) {$+$};
        \node [operator] (uadd) at (\xplus,\yupdate) {$+$};
        \node [operator] (iadd) at (\xplus,\yinput) {$+$};
        \node [operator] (oadd) at (\xplus,\youtput) {$+$};
        \node [operator] (sadd) at (\xinput,\yshort) {$+$};
        \node [operator] (fmul) at (\xforget,\yshort) {$\times$};
        \node [operator] (umul) at (\xinput,\yupdate) {$\times$};
        \node [operator] (omul) at (\xoutput,\youtput) {$\times$};
    
        \draw [->] (fadd) -- (fact);
        \draw [->] (uadd) -- (uact);
        \draw [->] (iadd) -- (iact);
        \draw [->] (oadd) -- (oact);
        
        \draw [->,rounded corners=\rounded] (fact) -| (fmul) node[label=fact.east] {forget gate};
        \draw [->] (fmul) -- (sadd);
        
        \draw [->,rounded corners=\rounded] (uact) -- (umul) node[label=uact.east] {update gate};
        \draw [->,rounded corners=\rounded] (iact) -| (umul) node[label=iact.east] {input value};
        
        \draw [->] (umul) -- (sadd);
        \draw [->,rounded corners=\rounded] (sadd) -| (sact);
        
        \draw [->] (oact) -- (omul) node[label=oact.east] {output gate};
        \draw [->] (sact) -- (omul) node[label=sact.south] {output value};
        
        \draw [-,rounded corners=\rounded] (sadd) -| (Com);
        \draw [-,rounded corners=\rounded] (Com) |- (Co);
        \draw [-] (Co) -- node[midway,below] {\normalsize Short Term Memory} (Ci);
        \draw [-,rounded corners=\rounded] (Ci) -| (Cim);
        \draw [->,rounded corners=\rounded] (Cim) |- (fmul);
        
        \draw [-,rounded corners=\rounded] (omul) -| (Hom);
        \draw [-,rounded corners=\rounded] (omul) -| (Ym);
        \draw [->,rounded corners=\rounded] (Ym) |- (Y);
        
        \draw [-,rounded corners=\rounded] (Hom) |- (Ho);
        \draw [-] (Ho) -- node[midway,above] {\normalsize Long Term Memory} (Hi);
        \draw [-,rounded corners=\rounded] (Hi) -| (Him);
        
        \draw [->,rounded corners=\rounded] (Him) |- (fadd.south west) node[xshift=-12pt,below] () {$C_f$};
        \draw [->,rounded corners=\rounded] (Him) |- (uadd.south west) node[xshift=-12pt,below] () {$C_u$};
        \draw [->,rounded corners=\rounded] (Him) |- (iadd.south west) node[xshift=-12pt,below] () {$C_i$};
        \draw [->,rounded corners=\rounded] (Him) |- (oadd.south west) node[xshift=-12pt,below] () {$C_o$};
        
        \draw [->,rounded corners=\rounded] (X) -- +(0:1) |- (fadd.north west) node[xshift=-12pt,above] () {$A_f$};
        \draw [->,rounded corners=\rounded] (X) -- +(0:1) |- (uadd.north west) node[xshift=-12pt,above] () {$A_u$};
        \draw [->,rounded corners=\rounded] (X) -- +(0:1) |- (iadd.north west) node[xshift=-12pt,above] () {$A_i$};
        \draw [->,rounded corners=\rounded] (X) -- +(0:1) |- (oadd.north west) node[xshift=-12pt,above] () {$A_o$};
        
        \draw [->] (faddb) -- (fadd.north);
        \draw [->] (uaddb) -- (uadd.north);
        \draw [->] (iaddb) -- (iadd.north);
        \draw [->] (oaddb) -- (oadd.north);
    
    \end{tikzpicture}
    }
    \caption{Feedforward (a) and LSTM (b) neuron architectures.}
    \label{fig:neurons}
\end{figure}
    
    The trainable parameters are $A_\cdot \in \R^{H \times Q}$, $b_\cdot \in \R^H$, and $C_\cdot \in \R^{H \times H}$, where $Q$ and $H$ denote the input and hidden layer sizes, respectively. The choice of $\tanh$ as the activation function in the Vanilla neuron, rather than more common alternatives such as ReLU, was motivated by the boundedness of the hedging strategy. In practice, this choice not only helped ensure bounded outputs but also led to faster convergence during training. For LSTM units, both the long-term and short-term memory states are initialized to zero at the beginning of each path.

    \paragraph{Architecture.} The neural networks used in our experiments follow a simple and consistent structure:
    \begin{enumerate}
        \item A trainable initial wealth parameter $X_0^\theta \in \R$,
        \item An input layer with $Q$ units and $L$ hidden layers, each with $H$ neurons of the same type (either Vanilla or LSTM),
        \item An output Vanilla neuron with sigmoid activation returning $\alpha^\theta$.
    \end{enumerate}
    The sigmoid activation at the output is appropriate because we are considering call options, for which the hedging strategies lie in a bounded range (typically $[0,1]$). Other derivatives, such as put options, may require different output activations due to different optimal hedge profiles.\\

    For the rest of the paper, networks will have an input layer, 2 hidden layers and be 10 neurons wide. The number of \textit{in-sample} simulations $I$ is set to 10,000 while the number of \textit{out-of-sample} (OOS) simulations is set to 20,000 to compute the mean and mean squared P\&L of the different methods. Moreover, simulations will have $J=126$ time-steps. Finally, training is done in batches of 64 for 64 epochs, amounting for 10,048 training steps, with a AdamW \citep{adamw} optimizer with weight decay 0.01 and a learning rate ranging from $10^{-2}$ to $10^{-3}$ with a cosine scheduler over the training steps. The algorithm used for training is summarized in Appendix~\ref{app:algorithms}.\\
    
    As a side note, we have also considered more advanced parallelizable models such as Transformers~\citep{transformer} and Structured State Space Models~\citep{sss_agu}, but opted for simpler architectures that can be trained on a standard laptop. LSTM remains a competitive baseline both in academic research and industrial applications~\citep{google_lstm, deepmind_lstm, openai_lstm}.

\subsection{Linear hedging with signature} \label{subsec:arribas} 

    The main idea of linear hedging with signatures is to restrict the set of admissible strategies $\alpha$ and initial wealth $X_0$, resp., to those that are linear functions of the signature of the time-augmented observable price path $\widehat{S}_t := (t, S_t)$ and its lead-lag counterpart resp. Specifically, one considers strategies of the form
    \begin{align} \label{eq:linearstrategy}
        \alpha_t^\textnormal{lin} \approx \bracket{\balpha}{\sig[t][S]}, \quad 0 \leq t < T,
    \end{align}
    with $\balpha \in \TA[2]$. The linear functional $\balpha$ can either be learned by gradient descent, as in the previous section, or, when the payoff can be approximated by a linear signature functional, explicitly computed by reducing the problem to a deterministic optimization problem, as described by \cite*{arribas-nonparam}. \\
    
    To proceed, we introduce the \textbf{Hoff lead-lag transform} \cite[Definition 2.1]{hoffleadlag}, which is crucial for combining Itô integrals with (Stratonovich) signatures. Let $X : [0,T] \to \R^d$ be a continuous semi-martingale discretely sampled at times $t_j$, for $j = 0, \dots, J$. The Hoff lead-lag transform $X^\LL : [0,T] \to \R^{2d}$ is a piecewise linear interpolation given by
    \begin{align*}
        X_t^\LL = (X_t^{\textnormal{lag}}, X_t^{\textnormal{lead}}),
    \end{align*}
    where for $t \in [t_j, t_{j+1})$, and $\Delta t = t_{j+1} - t_j$, we define:
    \begin{align*}
        (X_t^{\textnormal{lag}}, X_t^{\textnormal{lead}}) =
        \begin{cases}
            (X_{t_j}, X_{t_{j+1}}), & \text{if } t \in [t_j, t_{j + \frac{1}{2}}), \\
            \left(X_{t_j}, X_{t_{j+1}} + 4\frac{t - t_{j + 1/2}}{\Delta t} (X_{t_{j+2}} - X_{t_{j+1}})\right), & \text{if } t \in [t_{j + \frac{1}{2}}, t_{j + \frac{3}{4}}), \\
            \left(X_{t_j} + 4\frac{t - t_{j + 3/4}}{\Delta t} (X_{t_{j+1}} - X_{t_j}), X_{t_{j+2}}\right), & \text{if } t \in [t_{j + \frac{3}{4}}, t_{j+1}),
        \end{cases}
    \end{align*}
    and finally $X_{t_J}^\LL = (X_{t_{J-1}}, X_{t_J})$. Here, $t_{j+\frac{1}{2}}$ and $t_{j+\frac{3}{4}}$ denote the midpoints obtained via linear interpolation. This transformation allows us to express Itô integrals involving signature processes. For example, if $S$ is a semi-martingale and $\widehat S_t := (t, S_t)$, then
    \begin{align}\label{eq:arribas_itostrat}
        \int_0^T \bracketsig[t][S]{\balpha} \, \d S_t = \bracket{\balpha \word{4}}{\sig[T][S]^\LL},    
    \end{align}
    for $\balpha \in \TA[2]$, with $\balpha \word{4} \in \TA[4]$. Note that $\balpha \word{4}$ is sparse and $\sig[T][S]^\LL$ is typically evaluated over a finer (e.g., tripled) time grid. \\
    
    Suppose the payoff $\xi$ can be written or approximated by
    \begin{align} \label{eq:arribas_payoff}
        \xi^\textnormal{lin} \approx \bracket{\bxi}{\sig[T][S]^\LL},
    \end{align}
    for some $\bxi \in \TA[4]$. Then, plugging this expression into the optimal hedging condition \eqref{eq:hedgingoptimal}, we obtain by linearity:
    \begin{align} \label{eq:arribas_initwealth}
        X_0^\textnormal{lin} = \bracket{\bxi}{\E[\sig[T][S]^\LL]},
    \end{align}
    where $\E[\sig[T][S]^\LL]$ denotes the {expected signature} which can be estimated via Monte Carlo simulations or implied from market options. \\
    
    Concerning the hedging strategy, injecting the linear parameterizations \eqref{eq:linearstrategy} and \eqref{eq:arribas_payoff} into the objective criterion \eqref{eq:Jquad}, and using the shuffle product property of the signature together with \eqref{eq:arribas_itostrat}, leads to the equivalent deterministic optimization problem:
    \begin{align} \label{eq:arribas_optimization}
        \min_{\balpha \in \tTA[2]{M}} \bracket{ \left(X_0^\textnormal{lin} \emptyword + \balpha \word{4} - \bxi \right) \shupow{2} }{ \E[\sig[T][S]^\LL] },
    \end{align}
    where $\bxi \in \tTA[4]{(M+1)}$, $\balpha \in \tTA[2]{M}$, and $\E[\sig[T][S]^\LL] \in \tTA[4]{2(M+1)}$. This is now a deterministic (albeit possibly high-dimensional) optimization problem in $\balpha$, which can be solved using gradient descent techniques. These methods are studied in Section~\ref{sec:sig_reg}.

\section{Signatures as Features in Neural Networks} \label{sec:sig_nn}

    The goal of this section is to evaluate the impact of incorporating signature features into deep hedging strategies under two complementary axes:
    \begin{itemize}
        \item when the dynamics of the underlying asset exhibit {non-Markovianity}, and
        \item when the {payoff is path-dependent}, even under Markovian dynamics.
    \end{itemize}
    
    We consider a spot process $S_t$ with dynamics
    $$ \d S_t = S_t \Sigma_t \, \d B_t, $$

    where $(W, W^\perp)$ is a two-dimensional Brownian motion and $B = \rho W + \sqrt{1 - \rho^2} W^\perp$, with $\rho = -0.7$ fixed across all experiments. The stochastic volatility component $\Sigma_t$ is driven by $W$, and we consider two volatility models:
    
    \paragraph{Model 1: \cite{heston} (Markovian).} A classical Markovian stochastic volatility model where $\Sigma_t= \sqrt{V_t}$ with
    $$ \d V_t = \kappa (\theta - V_t) \, \d t + \eta \sqrt{V_t} \, \d W_t, \quad V_0 = \theta, $$
    
    with parameters $\kappa = 4$, $\theta = 0.0625$, and $\eta = 0.7$. The model admits a closed-form expression for the characteristic function, enabling pricing and hedging via Fourier methods for certain payoffs.
    
    \paragraph{Model 2: Shifted Fractional Bergomi (non-Markovian).} A non-Markovian model where the volatility is given by
    \begin{align} \label{eq:bergmimodel}
        \Sigma_t = \sqrt{f_0(t)} \exp \left( \frac{\varsigma}{2} \int_0^t (\epsilon + t - s)^{H - \frac{1}{2}} \, \d W_s - \frac{\varsigma^2}{8H} \left((\epsilon + t)^{2H} - \epsilon^{2H} \right) \right),
    \end{align}
    with parameters $f_0(t) \equiv 0.0625$, $\varsigma = 1.6$, $H = 0.1$, and $\epsilon = 1/52$. The shifted kernel enhances modeling flexibility, as shown in \cite{shaun_rough}. Unlike Model 1, Model 2 lacks an explicit characteristic function in the standard sense. However, it fits naturally within the signature volatility framework of Section~\ref{subsec:sighedging}, which opens the door to approximate Fourier pricing and hedging via signature methods (see Section~\ref{subsec:sighedging}). Indeed, we can represent the volatility process $\Sigma$ as an infinite series of the signature of the time-extended Brownian motion $\widehat{W}_t = (t, W_t)$. More precisely, using the representation formula for Gaussian Volterra process \cite*[Section 4.3]{linearfbm} and the shuffle product, we get that \eqref{eq:bergmimodel} can be re-written as
    \begin{align} \label{eq:lin_repr_fbm}
        \Sigma_t = \bracketsig{\bsigma_t}, \quad \bsigma_t &
        = \sqrt{f_0(t)} \shuexp{u_t \emptyword + \bm{k}_t \word{2}},
    \end{align}
    and
    \begin{align}
        \bm{k}_t
        := \frac{\varsigma}{2} \sum_{n=0}^\infty (\epsilon + t)^{H - \frac{1}{2} - n} \left( \tfrac{1}{2} - H \right)^{\bar{n}} \word{1} \conpow{n}, \quad u_t := \frac{\varsigma^2}{8H} \left(\epsilon^{2H} - (\epsilon + t)^{2H} \right),
    \end{align}
    and $(\cdot)^{\bar{n}}$ is the rising factorial.
    \begin{sqremark} \label{rem:shuexp}
        The shuffle exponential in \eqref{eq:lin_repr_fbm} could be naively computed as its $K$-truncated power series expansion around $a \in \R$, i.e.
        $$ \shuexp{\bell} \approx e^a \sum_{n=0}^K \frac{(\bell - a \emptyword) \shupow{n}}{n!}, \quad \bell \in \tTA{M}, $$
        where $K$ needs to be large enough to have satisfactory convergence with arbitrary $a$. However, the right choice of $a$ can lead to a simplified recursive formula for $\shuexp{\bell}$: take $a = \bell^\emptyword$, then we have
        $$ \shuexp{\bell} = e^{\bell^\emptyword} \emptyword + \shuexp{\bell} \succ \bell = e^{\bell^\emptyword} \emptyword + \sum_{\word{i} \in \alphabet} \left( \shuexp{\bell} \shuprod \bell \proj{i} \right) \word{i}, \quad \bell \in \tTA{M}. $$
        
        This better way to construct the shuffle exponential allows for only $M \times d$ shuffle products to get the true coefficients up to order $M$. Furthermore, in case of the shifted fractional Bergomi, one only needs $M$ shuffle products as only one projection is non-zero, namely against $\word{2}$, leading to only $M$ shuffle products, i.e.
        $$ \shuexp{u_t \emptyword + \bm{k}_t \word{2}} \underset{\leq M}{=} e^{u_t} \emptyword + \underbrace{\left( e^{u_t} \emptyword + \left( \cdots e^{u_t} \emptyword \cdots \shuprod \bm{k}_t \right) \word{2} \shuprod \bm{k}_t \right) \word{2}}_{\times M}. $$
    \end{sqremark}
    
    Throughout the section, a version of $\bsigma_t$ truncated at \textit{order 5} will be used for signature Fourier hedging, namely Sections~\ref{subsec:deep_euro} and \ref{subsec:deep_asian}. An odd truncation order and a non-positive correlation parameter $\rho$, ensures that the stock price $S$ in the truncated model is a true martingale as shown by \citep*{sig_martingale}. Figure~\ref{fig:euro_iv_fbergomi} shows the implied volatility of European put options priced with the signature Fourier approximation using the Riccati equation \eqref{eq:Ric} under shifted fractional Bergomi against Monte-Carlo prices of the starting model \eqref{eq:bergmimodel}. The prices obtained by Fourier fall into the 95\% Monte Carlo confidence interval showing that for relatively short maturities, the approximate Fourier technique is accurate enough.
    
    \begin{figure}[H]
        \centering
        \includegraphics[width=\twoplotswidth]{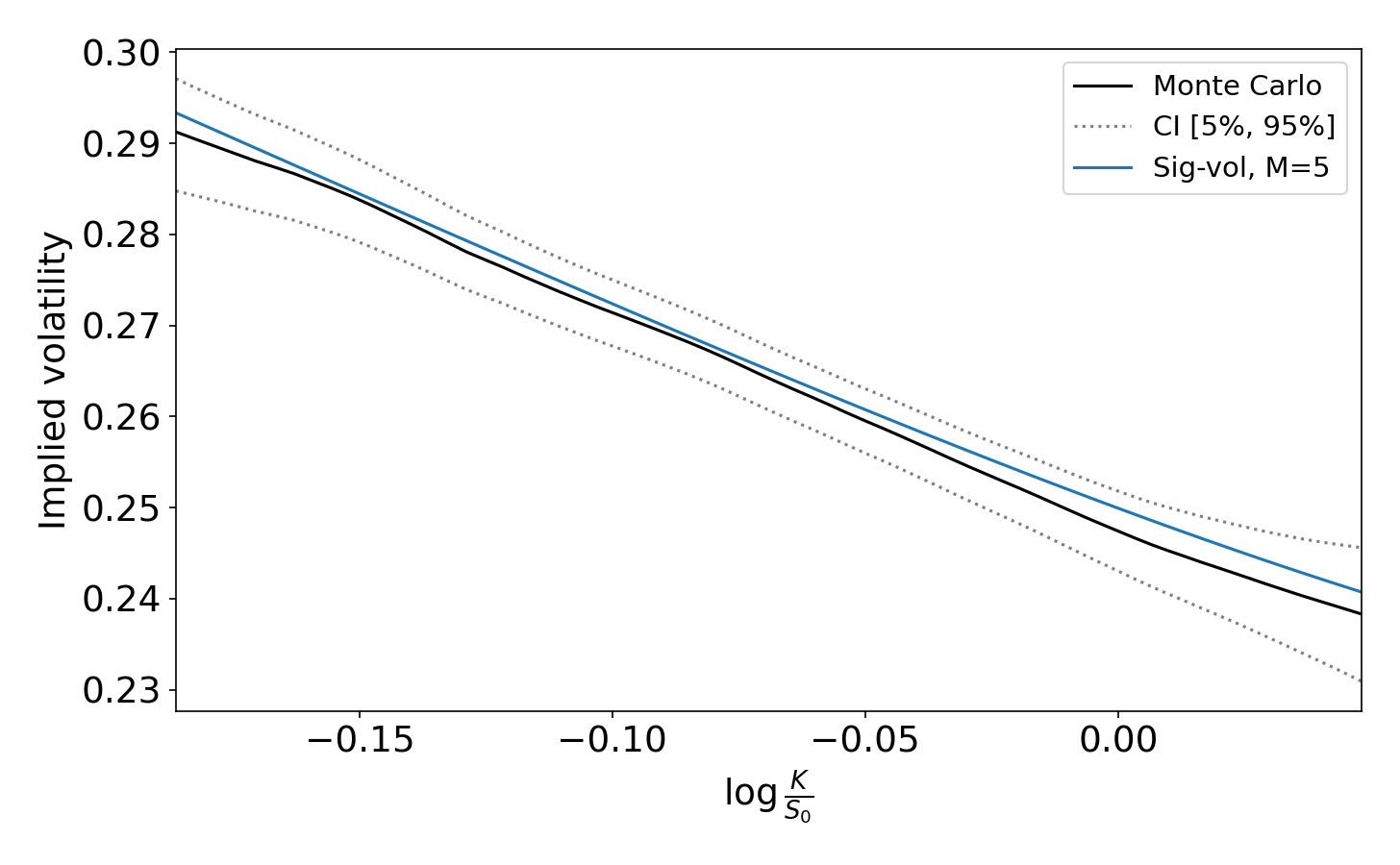}
        \caption{Implied volatility of a European put option under \fbergomi}
        \label{fig:euro_iv_fbergomi}
    \end{figure}
    
    \vspace{0.5em}
    \paragraph{Neural Networks.} For both models, we train and compare three neural network architectures to approximate the quadratic hedging strategy, as described in Section~\ref{subsec:deephedging}:
    \begin{itemize}
        \item \textbf{Vanilla NN (VNN):} A feedforward network with input $(t, S_t, \Sigma_t)$ for each $t$;
        \item \textbf{Signature NN (SNN):} A feedforward network with input given by the \emph{signature of order 4} of the path $(t, S_t, \Sigma_t)$ for each $t$;
        \item \textbf{Recurrent NN (RNN):} A recurrent LSTM-based network with input $(u, S_u, \Sigma_u)_{u \leq t}$ for each $t$.
    \end{itemize}
    
    The VNN serves as a baseline. The goal is to determine whether the explicit use of signature features allows simple feedforward networks (SNN) to outperform or match more complex recurrent architectures (RNN) on path-dependent tasks. Note that the signature can be computed quite efficiently on a CUDA GPU using \mbox{\textsc{Signatory}} of  \cite{signatory}. And even though we didn't have access to one, the time to compute them is negligible compared to the training time. \\
    
    Recall that all networks are 3 layers deep with 10 neurons wide. This choice of width was such that the 3 layer VNN would perform on par with Black-Scholes' solution in hedging an ATM European call option. Despite similar architecture, the parameter counts differ significantly: VNN has 272 trainable parameters, SNN has 1442 (due to the larger 120 signature input), and RNN has 2252 (due to its recurrent structure).
    
    \vspace{0.5em}
    \paragraph{Payoff Types.} To assess the contribution of pathwise information in hedging, we structure our experiments around three payoffs of increasing path-dependence:
    \begin{itemize}
        \item [\ref{subsec:deep_euro}.] \textbf{European Call Option:} a Vanilla payoff used to isolate the effect of non-Markovian dynamics
        \begin{align} \label{eq:euro_call}
            \left( S_T - K \right)^+,
        \end{align}
        
        \item [\ref{subsec:deep_asian}.] \textbf{Asian Call Option:} a payoff depending on the average price path
        \begin{align} \label{eq:asian_call}
            \left( \exp \left[ \frac{1}{T} \int_0^T \log S_t \d t \right] - K \right)^+,
        \end{align}
        
        \item [\ref{subsec:deep_look}.] \textbf{Look-back Call Option:} a path-dependent and strongly non-linear payoff
        \begin{align} \label{eq:lookback_call}
            S_T - \min_{0 \leq t \leq T} S_t.
        \end{align}
    \end{itemize}
    
    This structure allows us to systematically evaluate the role of signature features across both model complexity and payoff complexity.

\subsection{European option} \label{subsec:deep_euro}

    This first numerical experiment serves as a sanity check of our different methods. We here compare the performances of the three neural networks under a simplistic European call option \eqref{eq:euro_call} with strike $K=1$ and maturity $T$ of 6 months under Heston model and 1 month under shifted fractional Bergomi. \\

    We write \textit{Fourier, true} for the Fourier solution of hedging problem using Heston's characteristic function. On the other hand, we write \textit{Fourier, signature} for the Fourier approximation of that solution using Section~\ref{subsec:sighedging} and representation \eqref{eq:lin_repr_fbm} under shifted fractional Bergomi. Note that the difference in maturities has been chosen to make sure the benchmark signature Fourier approximation converges well enough, recall Figure~\ref{fig:euro_iv_fbergomi}. \\
    
    We can see in Figure~\ref{fig:euro_pdf} and Table~\ref{tab:euro_msp} that all three neural networks are on par and converge relatively well. It can however be noted that the Recurrent NN lags behind its feedforward counterparts, 4.5\% under Heston with heavier tails and 2.6\% under the shifted fractional Bergomi model even with low maturity. We believe it can be mainly attributed to overfitting as the \textit{in-sample} matched the Fourier solutions. It can also be argued that the Signature NN slightly outperforms the Vanilla NN. This leads us to believe that the structure of the LSTM is more prone to overfitting than simply its number of parameters. Recall the networks all have 3 layers of 10 neurons, with RNN, SNN and VNN respectively having 2252, 1442 and 272 trainable parameters. {Moreover, the Signature NN P\&L density in Figure~\ref{fig:euro_pdf}-(b) is closer to the Fourier solution than that of the Vanilla NN, showcasing the ability of the signature to capture the path-dependencies in the fractional Bergomi model.}
    
    \begin{figure}[H]
        \centering
        \subfloat[\centering \heston]{
        \includegraphics[width=\twoplotswidth]{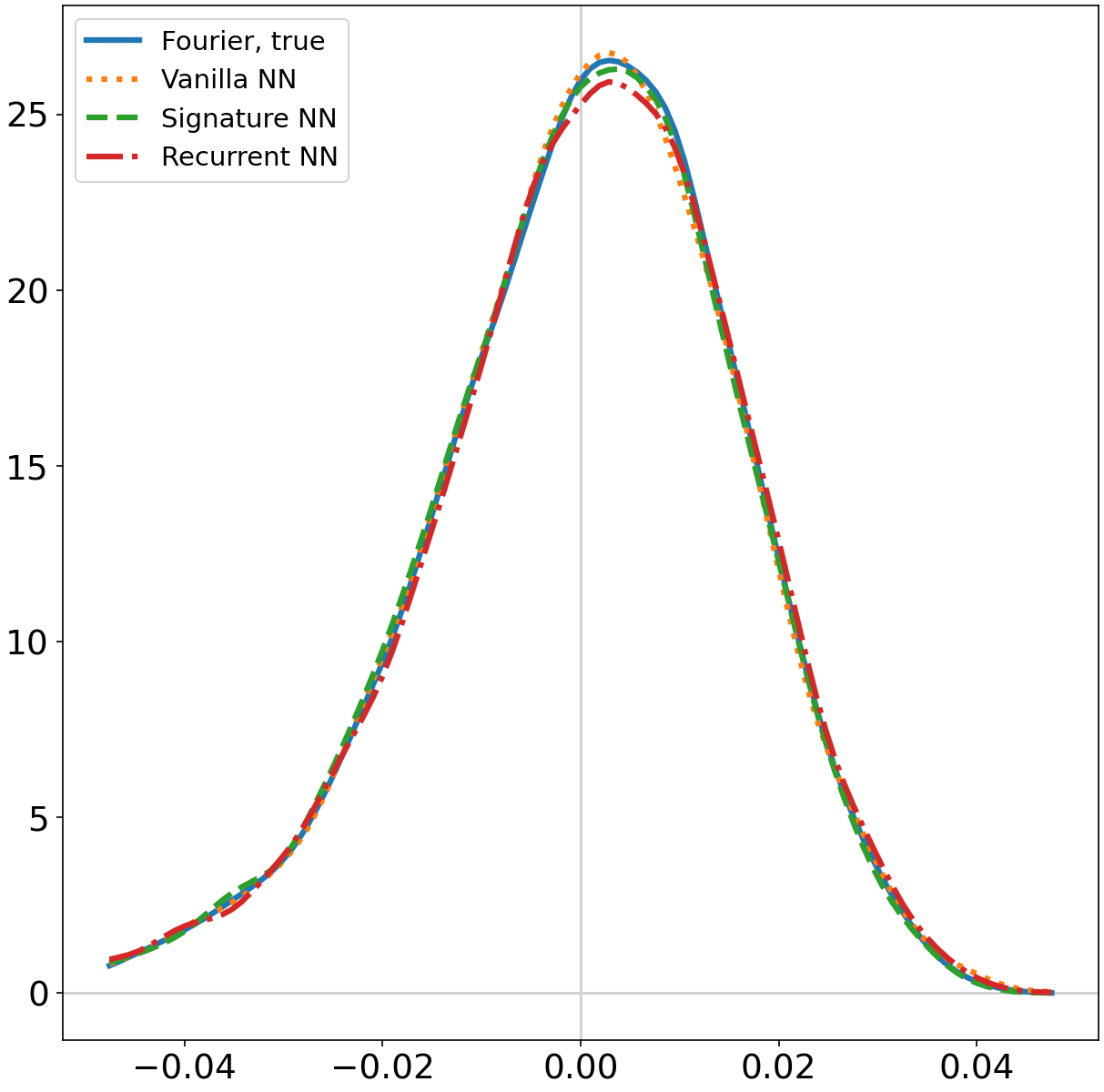}
        }
        \quad
        \subfloat[\centering \fbergomi]{
        \includegraphics[width=\twoplotswidth]{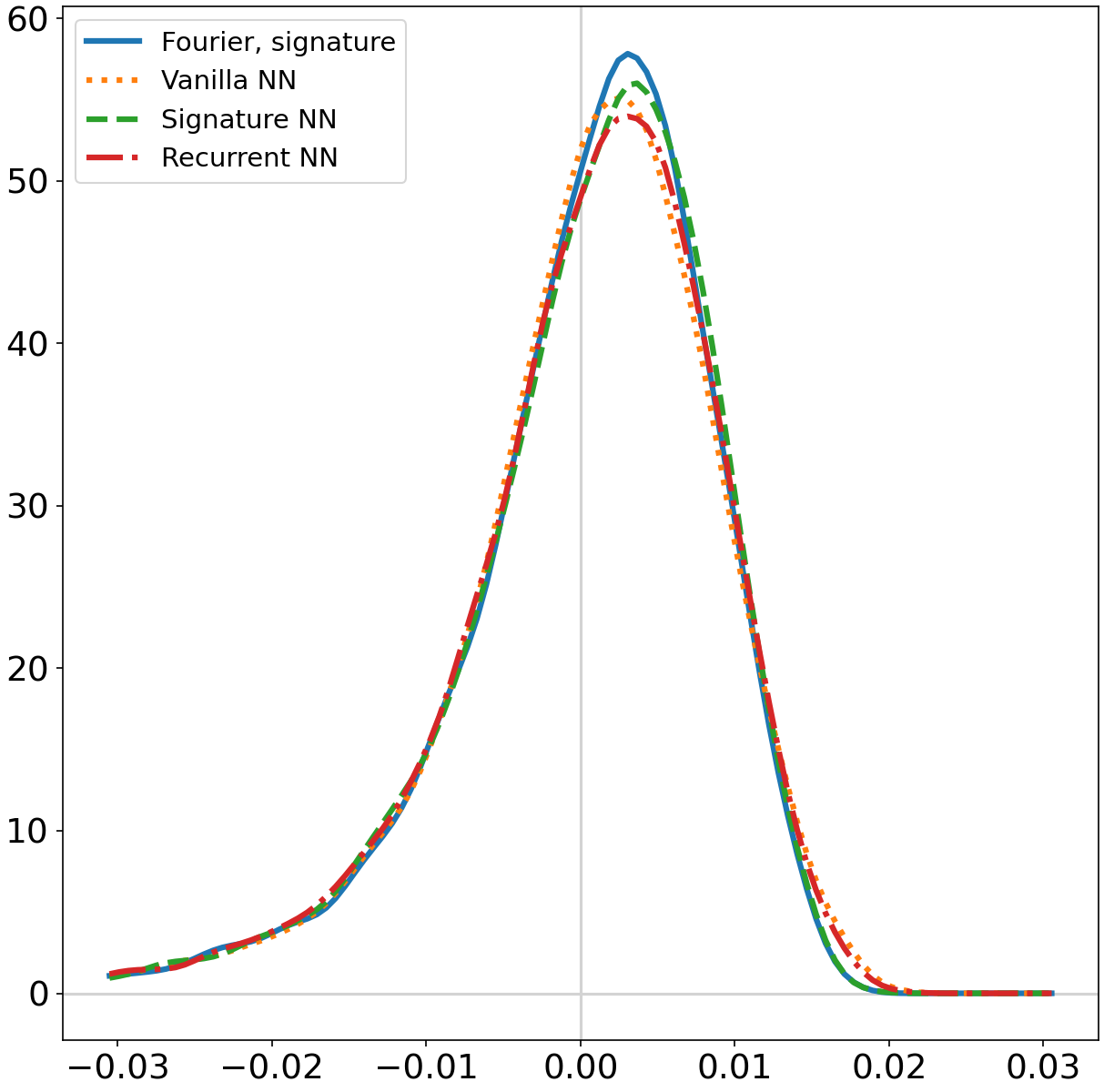}
        }
        \caption{\eurocall: \PDFcaption}
        \label{fig:euro_pdf}
    \end{figure}

    Note that the densities in Figures~\ref{fig:euro_pdf}, \ref{fig:asian_pdf} and \ref{fig:lookback_pdf} have been smoothed with a simple Gaussian KDE of \cite{kde_scott}.

    \begin{table}[H]
        \centering
        \subfloat[\centering \heston]{
        \begin{tabular}{|l|c|c|c|c|}
            \hline
                              & Fourier               & Vanilla NN            & Signature NN          & Recurrent NN          \\
            \hline
            Mean squared P\&L & $2.55 \cdot 10^{-4}$  & $2.59 \cdot 10^{-4}$  & $2.58 \cdot 10^{-4}$  & $2.67 \cdot 10^{-4}$  \\
            Mean P\&L         & $-6.42 \cdot 10^{-5}$ & $-1.70 \cdot 10^{-4}$ & $-3.82 \cdot 10^{-4}$ & $-2.42 \cdot 10^{-5}$ \\
            \hline
        \end{tabular}
        }
        \quad
        \subfloat[\centering \fbergomi]{
        \begin{tabular}{|l|c|c|c|c|}
            \hline
                              & Signature Fourier     & Vanilla NN            & Signature NN          & Recurrent NN          \\
            \hline
            Mean squared P\&L & $8.16 \cdot 10^{-5}$  & $8.20 \cdot 10^{-5}$  & $8.18 \cdot 10^{-5}$  & $8.38 \cdot 10^{-5}$  \\
            Mean P\&L         & $3.49 \cdot 10^{-5}$  & $1.45 \cdot 10^{-4}$  & $1.56 \cdot 10^{-4}$  & $1.23 \cdot 10^{-4}$  \\
            \hline
        \end{tabular}
        }
        \caption{\eurocall: \MSPcaption}
        \label{tab:euro_msp}
    \end{table}

    In addition to the mean squared P\&L, which is our metric of interest, we have included the mean P\&L in Table~\ref{tab:euro_msp}, \ref{tab:asian_msp} and \ref{tab:lookback_msp} to showcase the bias of the different model, i.e.~the accuracy of the initial wealth $X_0$. For example, even though the RNN performs worse than its feedforward counterparts in the case of hedging a European option, it does provide a better initial wealth. Recall that the networks are trained to output both the hedging exposure $\alpha$ over the horizon and the initial wealth $X_0$.

\subsection{Asian option} \label{subsec:deep_asian}
    
    We now turn to the at-the-money Asian call option \eqref{eq:asian_call}.
    Being out- or in-the-money did not make much difference in our experiments.\\
    
    Figure~\ref{fig:asian_pdf} and Table~\ref{tab:asian_msp} show that under both models the Signature NN (SNN) clearly outperforms its counterparts and converges relatively close to the Fourier solutions, 0.3\% and 1.1\% difference for the respective models. Moreover, although the Recurrent NN (RNN) performs better than the Vanilla NN (VNN) it is still quite far from the target solutions, 45\% and 7\% respectively. Moreover, we can see that the neural networks perform better overall under the shifted fractional Bergomi, but this is due to the shorter maturity.
    
    \begin{figure}[H]
        \centering
        \subfloat[\centering \heston]{
        \includegraphics[width=\twoplotswidth]{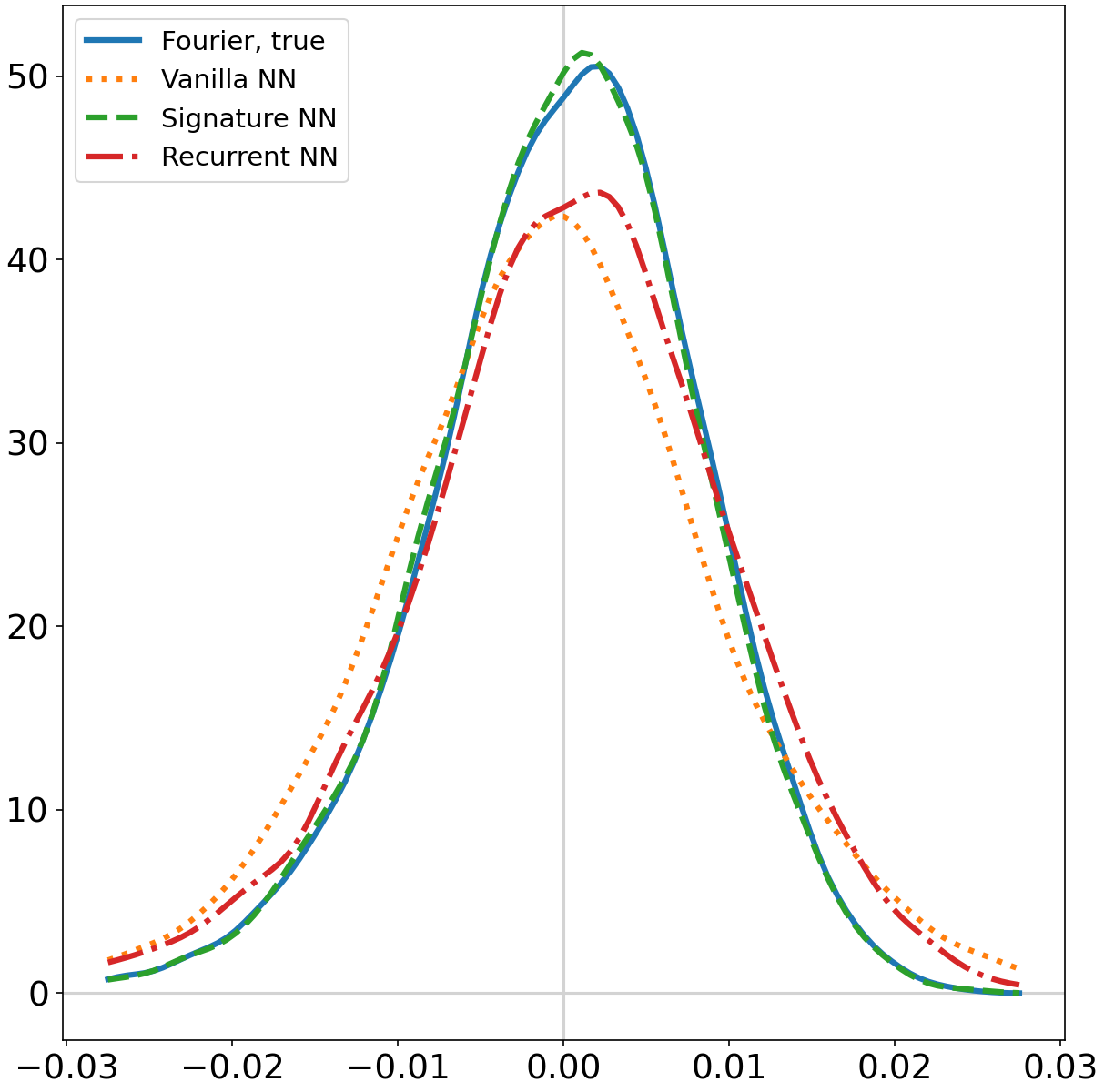}
        }
        \quad
        \subfloat[\centering \fbergomi]{
        \includegraphics[width=\twoplotswidth]{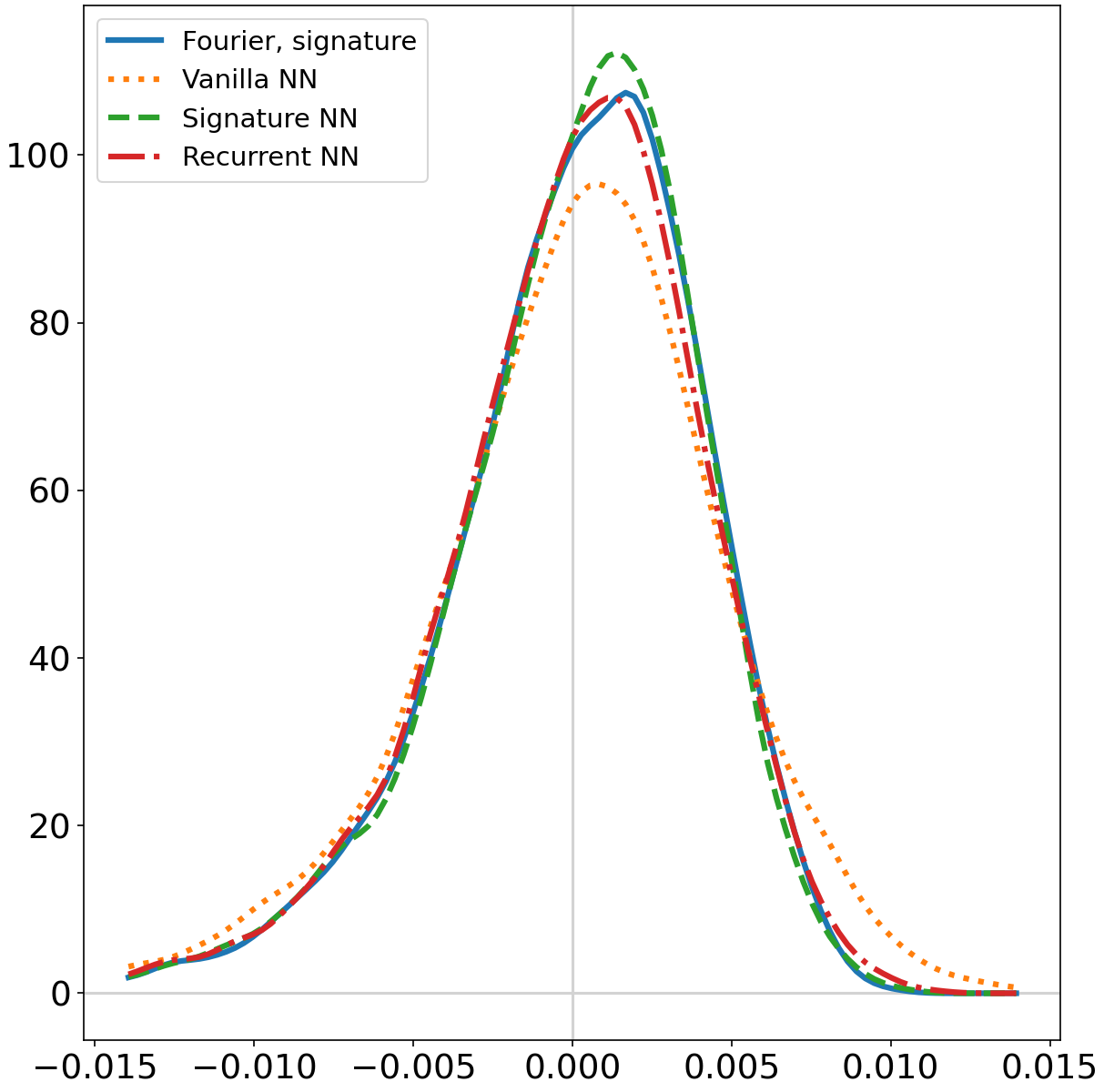}
        }
        \caption{\asiancall: \PDFcaption}
        \label{fig:asian_pdf}
    \end{figure}

    Figure~\ref{fig:asian_traj} showcases two sample paths (top) and the hedging ratio (bottom) of the different resolutions under Heston, Figure~\ref{fig:asian_traj}-(a), and under shifted fractional Bergomi, Figure~\ref{fig:asian_traj}-(b).

    \begin{table}[H]
        \centering
        \subfloat[\centering \heston]{
        \begin{tabular}{|l|c|c|c|c|}
            \hline
                               & Fourier               & Vanilla NN            & Signature NN          & Recurrent NN          \\
            \hline
            Mean squared P\&L  & $6.72 \cdot 10^{-5}$  & $1.14 \cdot 10^{-4}$  & $6.74 \cdot 10^{-5}$  & $9.77 \cdot 10^{-5}$  \\
            Mean P\&L          & $1.37 \cdot 10^{-5}$  & $-8.31 \cdot 10^{-4}$ & $-1.50 \cdot 10^{-4}$ & $5.86 \cdot 10^{-5}$  \\
            \hline
        \end{tabular}
        }
        \quad
        \subfloat[\centering \fbergomi]{
        \begin{tabular}{|l|c|c|c|c|}
            \hline
                               & Signature Fourier     & Vanilla NN            & Signature NN          & Recurrent NN          \\
            \hline
            Mean squared P\&L  & $1.71 \cdot 10^{-5}$  & $2.34 \cdot 10^{-5}$  & $1.73 \cdot 10^{-5}$  & $1.84 \cdot 10^{-5}$  \\
            Mean P\&L          & $-5.25 \cdot 10^{-5}$ & $-5.26 \cdot 10^{-5}$ & $-6.42 \cdot 10^{-5}$ & $2.78 \cdot 10^{-4}$  \\
            \hline
        \end{tabular}
        }
        \caption{\asiancall: \MSPcaption}
        \label{tab:asian_msp}
    \end{table}
    
    However, raw precision should be weighted by the speed at which it is achieved. In Table~\ref{tab:asian_time} we compare how long the different models took to achieve certain marks. As matter of comparison, the signature Fourier approximation took 1.9 seconds to compute the optimal hedging on the 20,000 OOS samples. This amounts for about 0.10 ms per sample path for hedging assuming you already have figured the signature volatility, i.e.~the representation $\bsigma_t$. We can first remark that the training step for both the VNN and SNN are similar even though the VNN has five times the number of trainable parameters. Moreover, the RNN is about forty times slower to take one training step than its feed-forward counterparts, although only ten (one and a half resp.) times larger than the VNN (SNN resp.). This is due to its recursive nature and not specifically to its number of trainable parameters at this scale, as showcased by the training steps for both SNN and VNN. Obviously, lowering the frequency at which we update the weights, i.e.~the number of time steps, would reduce the training time linearly. However one would need about 4 time steps to match the speed of the SNN and VNN, drastically reducing the hedging control of the RNN. \\

    Moreover, Table~\ref{tab:asian_time} also shows that not only does the SNN perform better than the VNN, but it also performs better \textit{faster}. Furthermore, something opposite can be said of the RNN, where its training steps take forty times longer but its total training time to reach the marks is more than forty times longer. This clearly shows how much the signature as input helps the neural network learn both faster and better.

    \begin{table}[H]
        \centering
        \subfloat[\centering \heston]{
        \begin{tabular}{|l|c|c|c|}
            \hline
                                          & Vanilla NN & Signature NN & Recurrent NN \\
            \hline
            Time per training step (avg.) & 1.03 ms    & 1.15 ms      & 54.7 ms      \\
            Time to $1.4 \cdot 10^{-4}$   & 3.08 sec   & 0.08 sec     & 66.2 sec     \\
            Time to $1.2 \cdot 10^{-4}$   & 5.48 sec   & 0.09 sec     & 119 sec      \\
            Time to $1 \cdot 10^{-4}$     & N/A        & 0.55 sec     & 250 sec      \\
            Time to $7 \cdot 10^{-5}$     & N/A        & 9.15 sec     & N/A          \\
            \hline
        \end{tabular}
        }
        \qquad
        \subfloat[\centering \fbergomi]{
        \begin{tabular}{|l|c|c|c|}
            \hline
                                          & Vanilla NN & Signature NN & Recurrent NN \\
            \hline
            Time per training step (avg.) & 1.28 ms    & 1.15 ms      & 49.6 ms      \\
            Time to $3.0 \cdot 10^{-4}$   & 8.20 sec   & 0.82 sec     & 211 sec      \\
            Time to $2.4 \cdot 10^{-4}$   & 12.4 sec   & 1.72 sec     & 336 sec      \\
            Time to $1.9 \cdot 10^{-4}$   & N/A        & 4.24 sec     & 434 sec      \\
            Time to $1.8 \cdot 10^{-4}$   & N/A        & 7.76 sec     & N/A          \\
            \hline
        \end{tabular}
        }
        \caption{\asiancall: \TIMEcaption}
        \label{tab:asian_time}
    \end{table}

    A more thorough examination of training time for the SNN in terms of its truncation order and the use or not of the log transform of the signature has been detailed in \ref{app:sig_order_sup2}. \\

    Finally, Figure~\ref{fig:asian_traj} showcases two sample paths (top) and the hedging ratio (bottom) of the different resolutions under Heston, Figure~\ref{fig:asian_traj}-(a), and under shifted fractional Bergomi, Figure~\ref{fig:asian_traj}-(b). The strategies produced by the SNN are the closest to the reference ones computed via Fourier methods.
    
    \begin{figure}[H]
        \centering
        \subfloat[\centering \heston]{
        \includegraphics[width=\oneplotwidth]{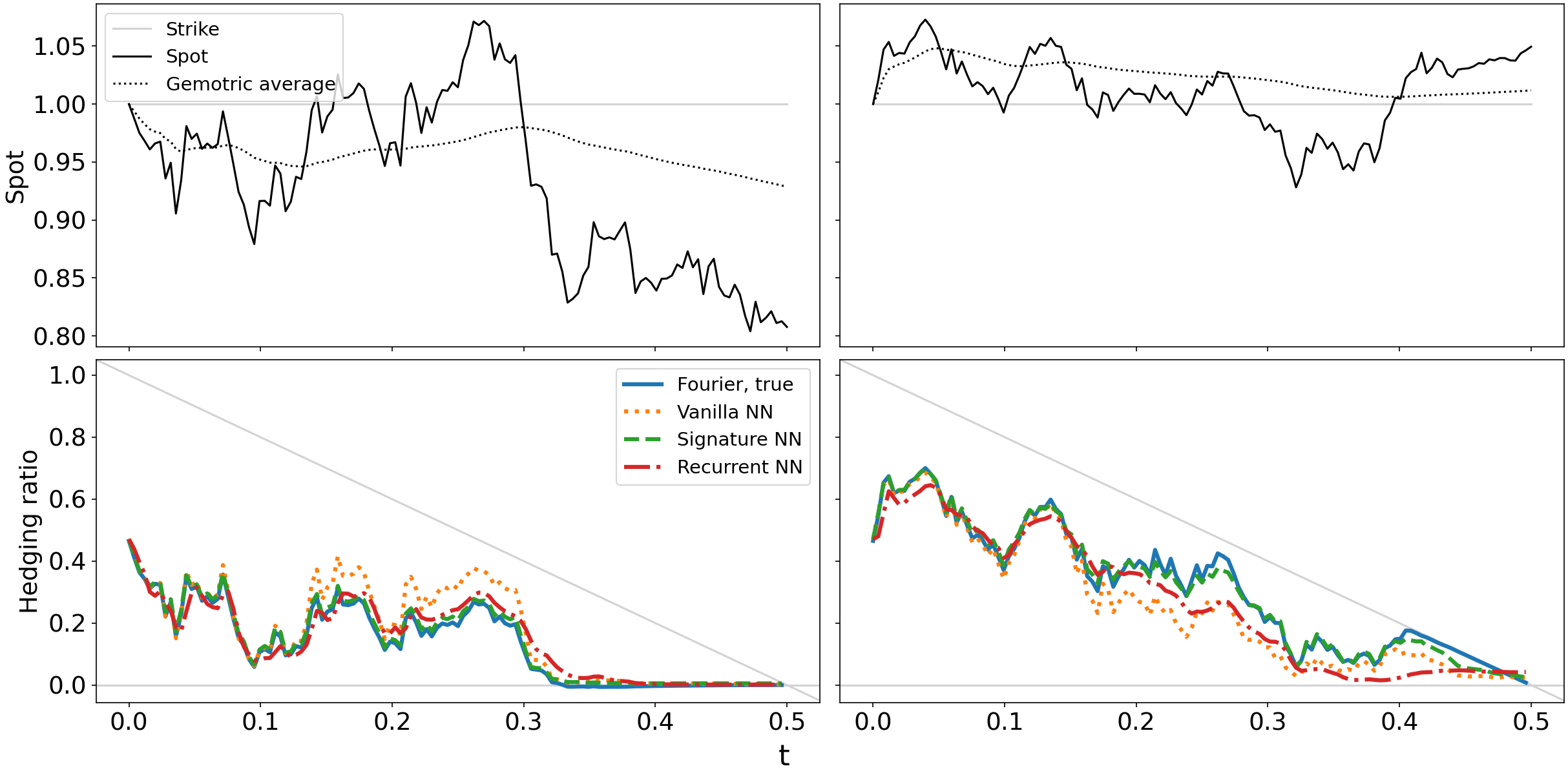}
        }
        \quad
        \subfloat[\centering \fbergomi]{
        \includegraphics[width=\oneplotwidth]{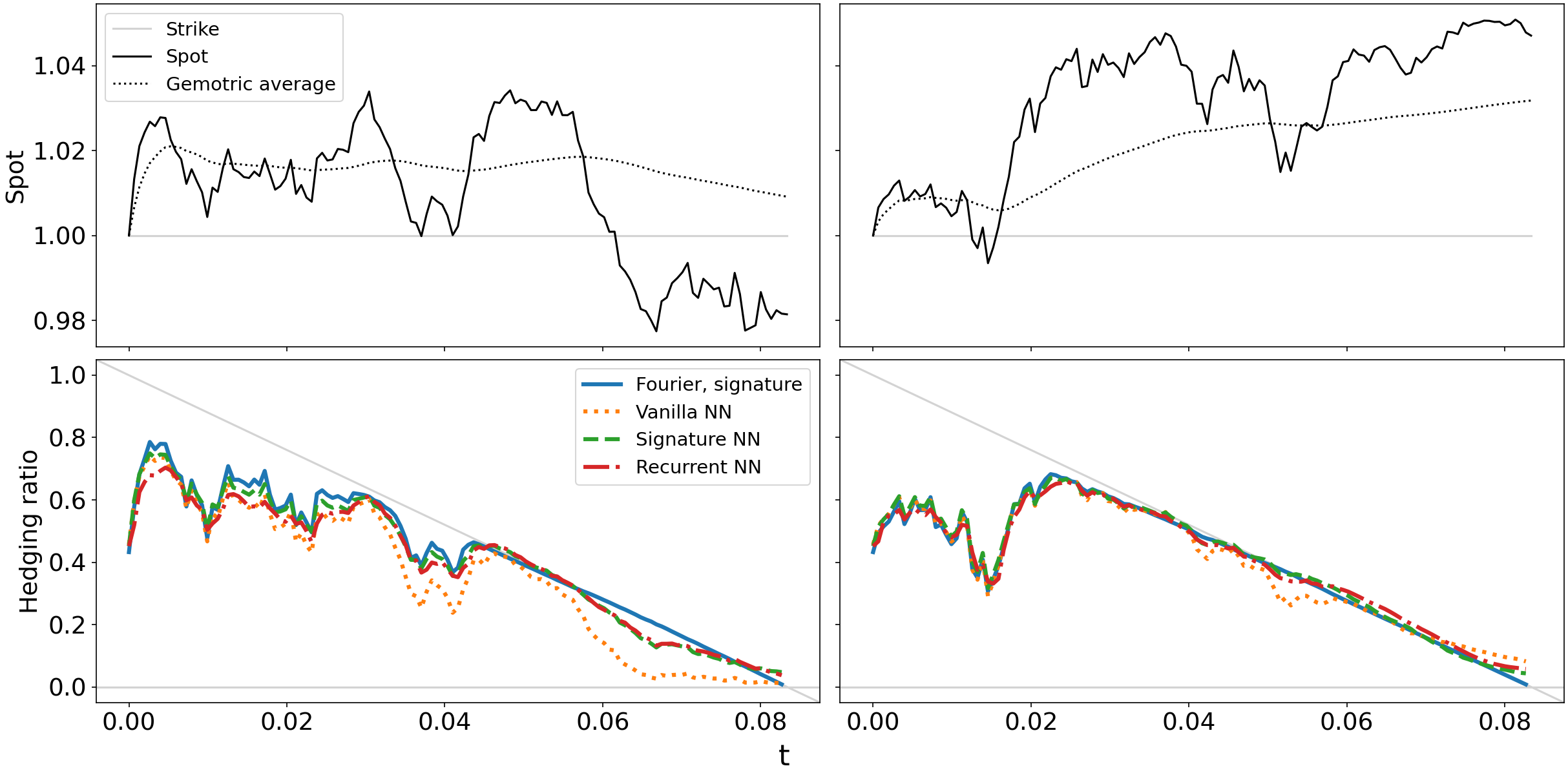}
        }
        \caption{\asiancall: \TRAJcaption}
        \label{fig:asian_traj}
    \end{figure}

\subsection{Look-back option} \label{subsec:deep_look}

    Shifting our attention to the lookback call option with floating strike \eqref{eq:lookback_call}, Figure~\ref{fig:lookback_pdf} and Table~\ref{tab:lookback_msp} make it is first clear that the Vanilla NN (VNN) under-performs other architectures, even the Black-Scholes naive resolution under both models. On the contrary, the Recurrent NN (RNN) now over-performs even the Signature NN (SNN), by 11\% and 5\% for the Heston and shifted fractional Bergomi models respectively.
    
    \begin{figure}[H]
        \centering
        \subfloat[\centering \heston]{
        \includegraphics[width=\twoplotswidth]{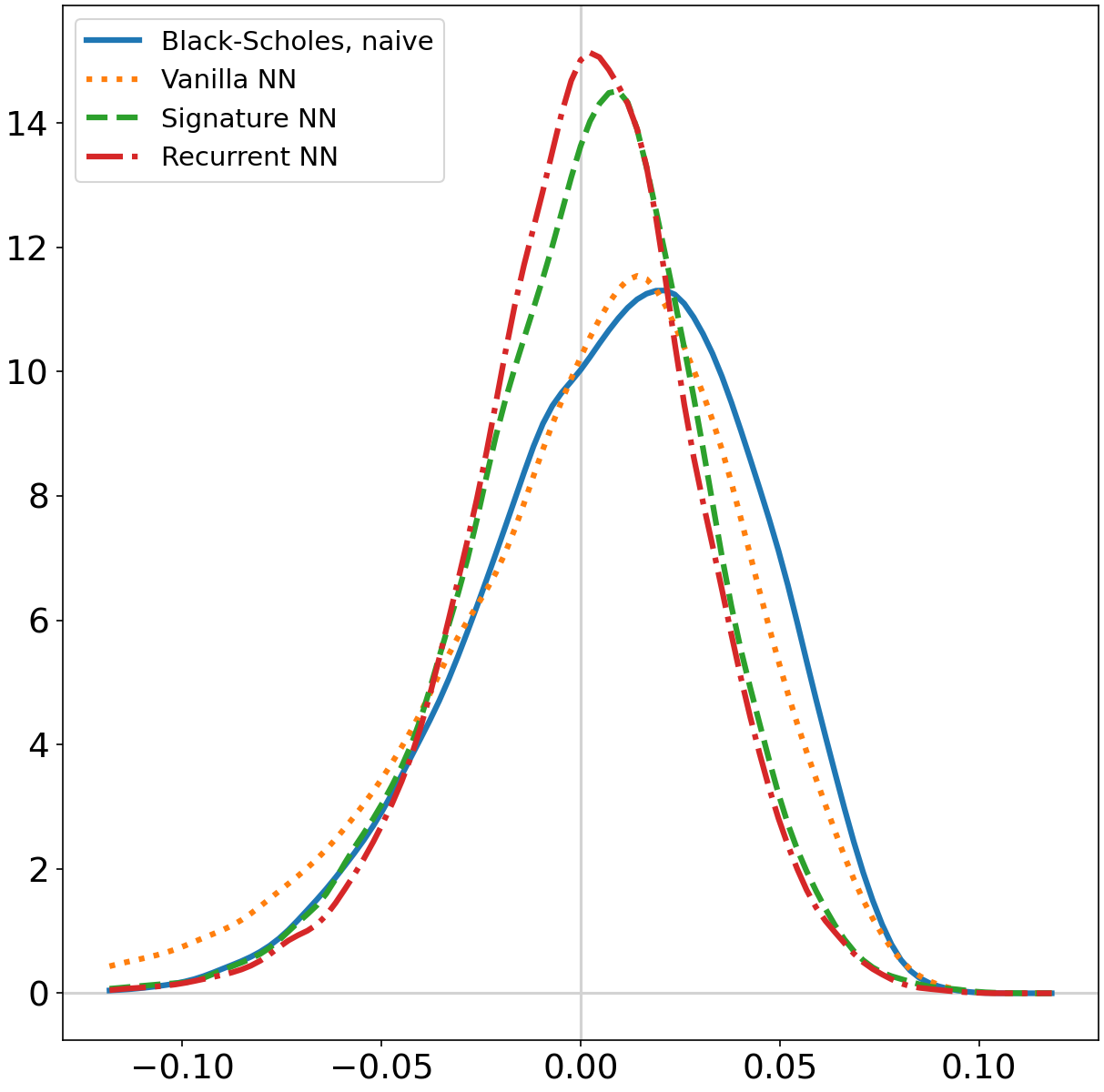}
        }
        \quad
        \subfloat[\centering \fbergomi]{
        \includegraphics[width=\twoplotswidth]{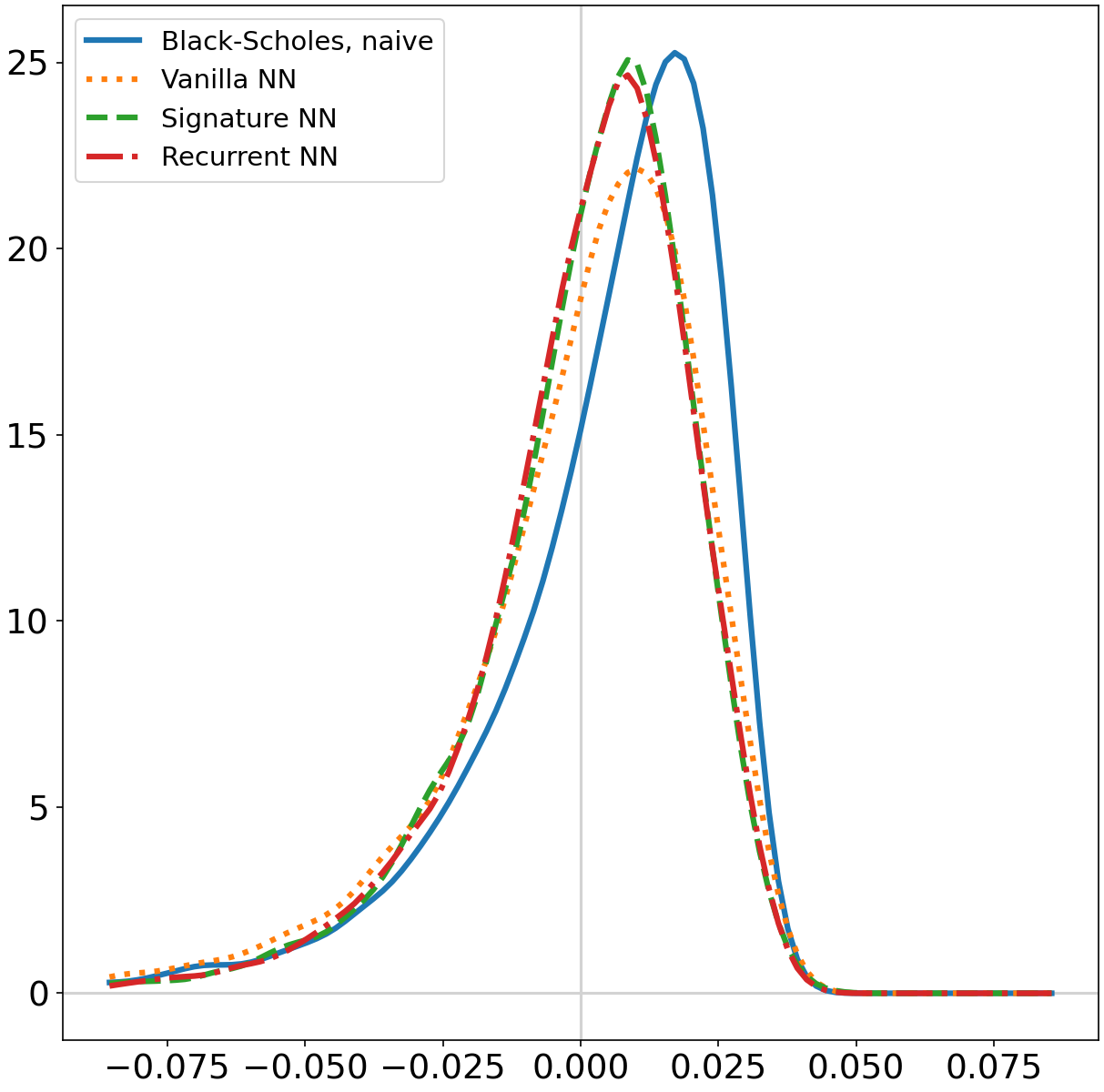}
        }
        \caption{\lookbackcall: Out-of-sample P\&L density of the delta hedging associated with the Black-Scholes naive solution (blue), Vanilla NN (orange), Signature NN (green) and Recurrent NN (red).}
        \label{fig:lookback_pdf}
    \end{figure}

    \begin{table}[H]
        \centering
        \subfloat[\centering \heston]{
        \begin{tabular}{|l|c|c|c|c|}
            \hline
                              & Naive Black-Scholes   & Vanilla NN            & Signature NN          & Recurrent NN          \\
            \hline
            Mean squared P\&L & $1.24 \cdot 10^{-3}$  & $1.61 \cdot 10^{-3}$  & $9.21 \cdot 10^{-4}$  & $8.18 \cdot 10^{-4}$  \\
            Mean P\&L         & $8.05 \cdot 10^{-3}$  & $1.85 \cdot 10^{-4}$  & $3.80 \cdot 10^{-4}$  & $4.30 \cdot 10^{-6}$  \\
            \hline
        \end{tabular}
        }
        \quad
        \subfloat[\centering \fbergomi]{
        \begin{tabular}{|l|c|c|c|c|}
            \hline
                              & Naive Black-Scholes   & Vanilla NN            & Signature NN          & Recurrent NN          \\
            \hline
            Mean squared P\&L & $5.32 \cdot 10^{-4}$  & $6.08 \cdot 10^{-4}$  & $4.63 \cdot 10^{-4}$  & $4.39 \cdot 10^{-4}$  \\
            Mean P\&L         & $4.00 \cdot 10^{-3}$  & $-8.74 \cdot 10^{-4}$ & $6.46 \cdot 10^{-5}$  & $1.42 \cdot 10^{-4}$  \\
            \hline
        \end{tabular}
        }
        \caption{\lookbackcall: Out-of-sample mean and mean squared P\&L of the delta hedging associated with the Black-Scholes naive solution (blue), Vanilla NN (orange), Signature NN (green) and Recurrent NN (red).}
        \label{tab:lookback_msp}
    \end{table}

    Again, a more thorough examination of training times and performances has been detailed in Appendix \ref{app:sig_order_sup2} for the interested reader. \\
    
    Finally, we also included Figure~\ref{fig:lookback_traj} below to help visualize beyond the P\&L densities. We can first remark that the VNN is very smooth compared to the other methods. On the contrary, Black-Scholes resolution is much more volatile and seems to consistently overestimate the hedging ratio. Finally, the SNN and RNN are the closest, at certain times the SNN seems to react stronger than the RNN leading to higher highs and lower lows.
    
    \begin{figure}[H]
        \centering
        \subfloat[\centering \heston]{
        \includegraphics[width=\oneplotwidth]{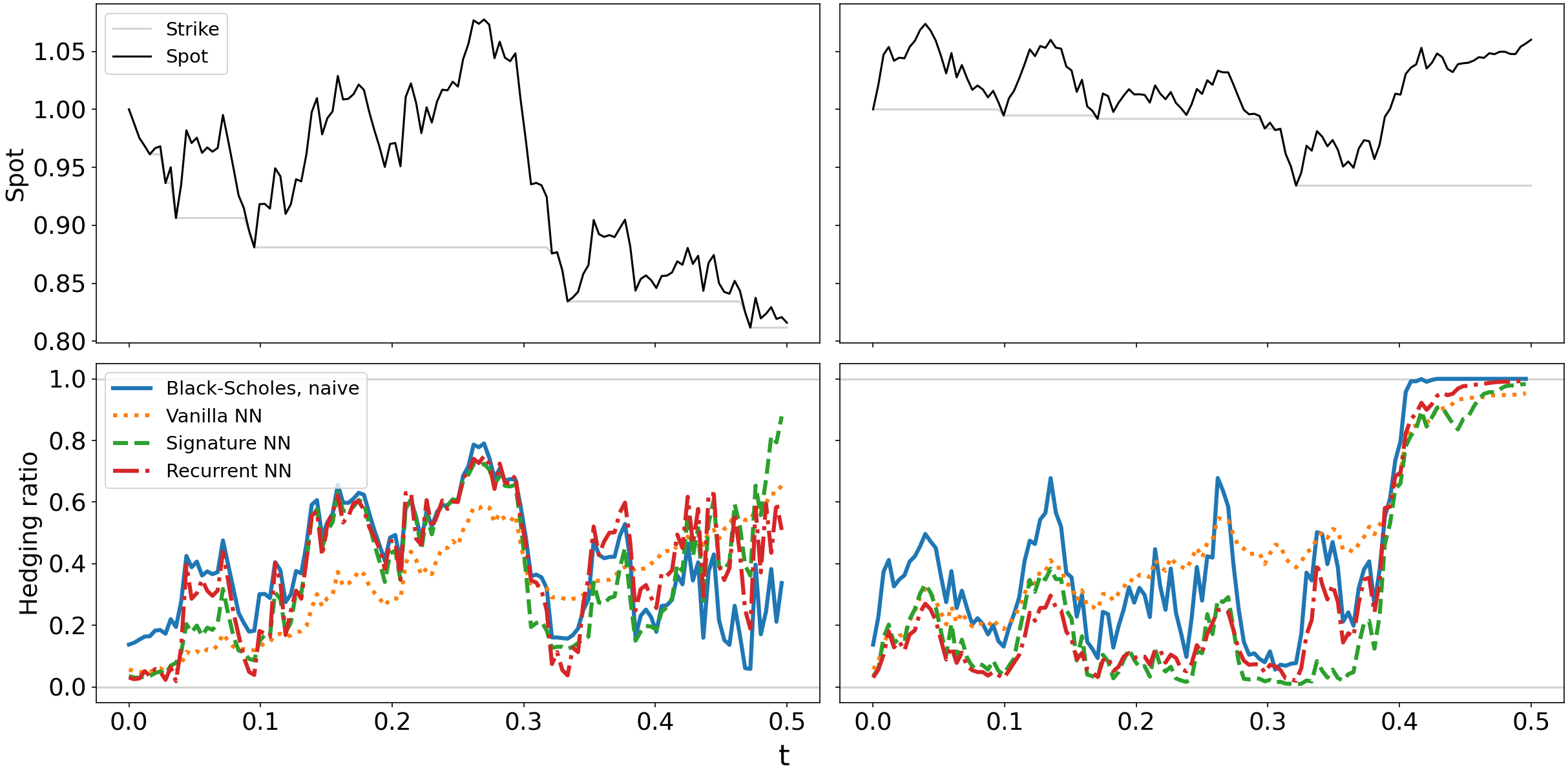}
        }
        \quad
        \subfloat[\centering \fbergomisix]{
        \includegraphics[width=\oneplotwidth]{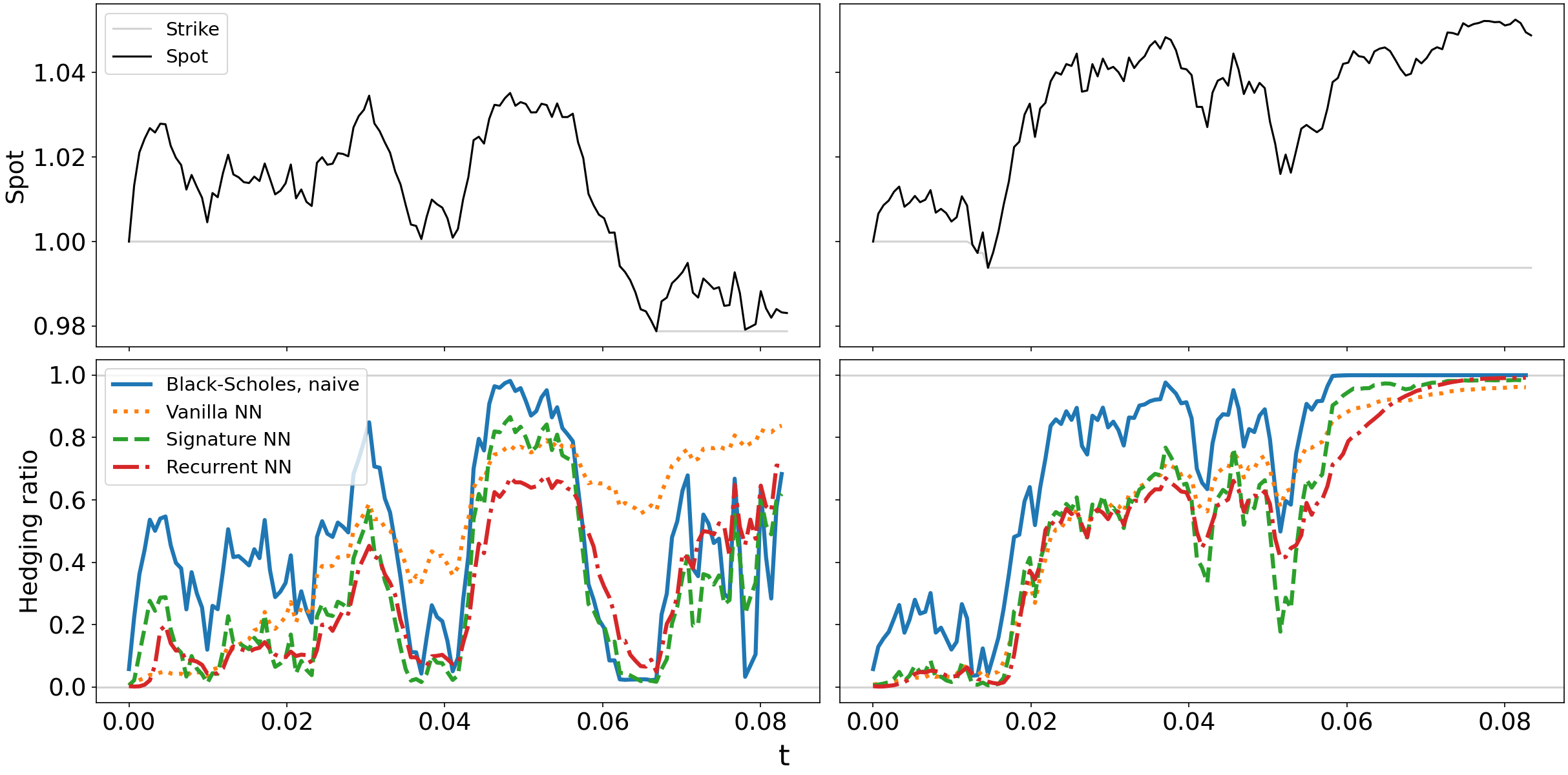}
        }
        \caption{\lookbackcall: Two out-of-sample spot trajectories (top) and the delta hedging (bottom) associated with the Black-Scholes naive solution (blue), Vanilla NN (orange), Signature NN (green) and Recurrent NN (red).}
        \label{fig:lookback_traj}
    \end{figure}
    
    We did not include the timings for the look-back call as they are most similar to Table~\ref{tab:asian_time} on the Asian call option and its conclusions hold for the look-back option with the SNN being much faster to train than the RNN. \\

    In conclusion, the payoff considered here is highly non-linear and path-dependent. While the signature captures significantly more information than the VNN, resulting in a notable improvement, it still falls slightly short of the RNN, which appears better able to extract the relevant pathwise information. However, when accounting for computational cost, one could argue that the SNN performs on par with, or even favorably compared to, the RNN.

\section{Signatures as linear regressors} \label{sec:sig_reg}
    
    The goal of this section is to compare two hedging methods based on linear functions of the signatures of observables.
    {Both approaches use an estimate of the expected signature. The spot price $S_t$ for the linear method in Section~\ref{subsec:arribas}, and the spot volatility $\Sigma_t$ for the Fourier method in Section~\ref{subsec:sighedging}.}
    We assume throughout this section that we are given such data (either exactly or through Monte Carlo simulations). Our goal is to extract, from these observations, an approximate pricing function and hedging strategy. \\
    
    More precisely:
    \begin{itemize}
        \item One approach directly learns the initial wealth $X_0$ and hedging strategy $\alpha$ as a linear functional of the  path signature of the underlying price as described in Section~\ref{subsec:arribas}.
        
        \item The other approach consists in calibrating a {signature volatility model}  as in \eqref{eq:sigvol} so that it reproduces the observed expected signature. More precisely, one learns the coefficients $\bsigma$ such that the model volatility $\Sigma_t = \langle \bsigma, \sig[t][W] \rangle$ (with $W$ a Brownian motion) yields a model-consistent $\sigE[t][Y]$ where $\widehat{Y}_t = (t, \Sigma_t)$. Once the model is calibrated, a numerical method, e.g., Fourier inversion, can be used to compute the initial wealth $X_0$ and hedging strategy $\alpha$. See Section~\ref{subsec:sighedging}.
    \end{itemize}
    
    The purpose of this section is to investigate which approach performs better under various data and model assumptions. Conceptually, we aim to answer the following question:
    
    \begin{quote}
        \itshape
        Is it more effective, given observations of $S_t$ and $\Sigma_t$, to directly {linearize the hedging strategy} (using the signature of the observed spot price), or to linearize the volatility dynamics (using the signature of a Brownian motion), and then use model-based methods to derive prices and strategies?
    \end{quote}
    
    \paragraph{Methods.} We benchmark the following approaches, comparing not only each method with the other, but also against their respective \emph{best-case scenarios} (denoted BCS), i.e., situations where full information and computational resources are available.
    
    \begin{itemize}
        \item \label{item:arribas_methods}
        \textbf{Linear signature methods}, where both initial wealth and hedging strategy are expressed as a linear functional of $\sig[t][S]$:
        \begin{enumerate}[label=(\alph*)]
            \item \textbf{Linear Regression-Expectation (Linear REG):} Learn the expected payoff by regressing on the signature of $S$ using samples of $\E[\sig[t][S]]$, via problem~\eqref{eq:arribas_optimization}. See Section~\ref{subsec:arribas}.
            
            \item \textbf{Linear Reinforcement-Learning (Linear BCS):} Learn directly the hedging strategy by reinforcement optimization over trading P\&L, using access to full model simulations. This gives the best linear signature strategy achievable with enough data and compute, bypassing the linearization of payoff and computation of the expected signature altogether.
        \end{enumerate}
        
        \item
        \textbf{Signature volatility model methods}, where the volatility process is $\Sigma_t = \bracketsig{\bsigma}$ and option prices are computed in the model:
        \begin{enumerate}[label=(\alph*)]
            \item \textbf{Signature Fourier Regression (Fourier REG):}
            Retrieve a truncated approximation of the linear volatility $\hat{\bsigma}$ from data and then use Section~\ref{subsec:sighedging} to recover the optimal strategy via Fourier inversion (when available).
            In order to retrieve $\hat{\bsigma}$, it suffices to compute the expected signature of the stock volatility process under the signature volatility model. This proceeds in two steps:
            \begin{enumerate}[label=(\roman*)]
                \item In a first step, one computes the (truncated) signature elements of the time-augmented process $\Sigma$, which are explicit functions of the signature of the driving Brownian motion. This is detailed for instance in \cite[Theorem 3.16]{christa_theo_calib}. More precisely, for a volatility given by $\Sigma_t = \bracketsig{\bsigma}$ and $\widehat{Y}_t = (t, \Sigma_t)$, its signature $\sig[\cdot][Y]$ satisfies
                $$ \sig[\cdot][Y]^{\word{v}} = \bracketsig[\cdot]{\bell(\bsigma, \word{v})} \quad \text{for all words } \word{v} \in V, $$
                where $\bell: \TA \times V \to \TA$ is a deterministic function mapping the model parameters and a word to a set of signature coefficients. This allows one to express each coordinate of $\sig[\cdot][Y]$ as a linear functional of the Brownian signature.
        
                \item In a second step, calibration reduces to a simple reverse optimization problem. Using \cite{fawcett} formula, which gives closed-form expressions for the expected signature of a Brownian motion, one can match the model-implied expected signature of $\widehat{Y}_t = (t, \Sigma_t)$ to the observed values $\E[\sig[t][Y]]$, and solve for the optimal coefficients $\bsigma$. This regression-like procedure is described in Appendix~\ref{app:algorithms}.
            \end{enumerate}
            
            \item
            \textbf{Signature Fourier Representation (Fourier BCS):} Use the true model parameters (oracle setting), truncated at higher order, as a benchmark for the best achievable performance in the model class.
        \end{enumerate}
    \end{itemize}
    
    We emphasize that these methods are not always comparable: Fourier-based approaches rely on specific assumptions about the price dynamics (e.g., geometric Brownian motion with stochastic volatility) and payoff regularity (e.g., Fourier-invertible). By contrast, the linear signature regressions are agnostic to the underlying dynamics and can be applied in more general contexts. In this section, we restrict our attention to settings where both methods are applicable to allow for a meaningful comparison.

  \begin{sqremark} \label{rmk:sigvolreg}
     Other alternatives exist although they will not be studied in this paper. The first notable mention, discussed in \cite[Section 5.3]{sigvolfourier}, calibrates $\hat{\bsigma}$ on option prices across strikes and maturities. Note that a similar approach could be to calibrate $\hat{\bsigma}$ to minimize the mean square hedging P\&L over some observations. The second notable mention, discussed in \cite[Example 4.1]{christa_theo_calib}, uses higher frequency spot price data to estimate the spot covariation and then recover the underlying Brownian increments, allowing a simple regression to recover the coefficients.
    \end{sqremark}

    \paragraph{Payoff Types.} We compare both methods on three payoff functions with gradual complexity:
    \begin{itemize}
        \item [\ref{subsec:arribas-poly}.] \textbf{Polynomial payoff:} a payoff that is linear in the truncated signature
        \begin{align} \label{eq:poly_call_arribas}
            \xi = \sum_{k=0}^4 c_k S_T^k,
            \quad \text{with }
            c = \left\{ -\tfrac{1}{16}, \tfrac{1}{2}, -\tfrac{7}{8}, \tfrac{1}{2}, -\tfrac{1}{16} \right\},
        \end{align}
        where $S$ is a geometric Brownian motion in order to have an explicit solution to compare the methods to.
        
        \item [\ref{subsec:arribas-euro}.] \textbf{European Call Option:} a payoff to showcase the first limitations of signature approximations of non-linear payoffs:
        \begin{align} \label{eq:euro_call_arribas}
            \left( S_T - K \right)^+,
        \end{align}
        where $S$ follows the Heston model with its variance $V_t=\Sigma_t^2$ having CIR dynamics \eqref{eq:CIR_arribas}. This allows us to compare the methods to the Fourier solution.
        
        \item [\ref{subsec:arribas-asian}.] \textbf{Asian Call Option:} a path-dependent payoff to further test signature methods
        \begin{align} \label{eq:asian_call_arribas}
            \left( \exp \left[ \frac{1}{T} \int_0^T \log S_t \d t \right] - K \right)^+,
        \end{align}
        where $S$ will be geometric with \textit{delayed-equation} (DE) volatility \eqref{eq:DE_arribas}. In this model we can compute the true truncated signature representationfor the volatility and hence compare Fourier REG with its \textit{best-case-scenario} counterpart.
    \end{itemize}

    \begin{sqremark}
        Note that in order to alleviate the influence of the starting point for both optimizations in Linear REG and Fourier REG, we compute the gradient descent over 10 different initial guesses and at the end of training select the solution with lowest \textit{training} loss. Furthermore, every computation is linear and can be broadcasted, so in the end it doesn't impact at all the duration of training. However, it did improve significantly the found local minima. Initial guesses were drawn at random around 0 for Linear REG and around a Ornstein-Uhlenbeck representation with maximum likelihood parameters \citep{maxlikelihood_oucir} for Fourier REG.
    \end{sqremark}

\subsection{Polynomial Payoff under Black-Scholes} \label{subsec:arribas-poly}
    
    This first numerical experiment is done in a very simplistic case to test the framework and, eventually, remark some artifact that was not specifically expected. We simulate geometric Brownian motions $S$ for the spot price with constant volatility $\sigma = 0.25$ and no drift, and compare Linear REG to its \textit{best-case-scenario} counterpart Linear BCS, against the polynomial payoff \eqref{eq:poly_call_arribas} with maturity 6 months $T=1/2$. \\

    This experiment was initially only intended for checking the implementations, as both the linear payoff \eqref{eq:arribas_payoff} and the linear strategy \eqref{eq:linearstrategy} can be found explicitly. However, as can be seen in Figure~\ref{fig:polyarribas-blackscholes}, the Linear BCS comes very close to the true solution, about 1.7\% higher mean square P\&L, while the Linear REG is 23\% higher. The linear (ridge) regression in \eqref{eq:arribas_payoff} did get quite close to the true weights, as can be expected, but the clear bias in the P\&L shows that the problem mostly lies in the Monte Carlo approximation of the expected signature and the consequent optimization~\eqref{eq:arribas_optimization}. To make sure, as the true linear payoff is explicitly
    $$ \bxi = \sum_{m=0}^M \sum_{k=0}^4 c_k \binom{k}{m} S_0^{k-m} \word{2} \conpow{m}, $$
    
    where $M \geq 4$ is the truncation order of $\bxi$, we have also optimized $\hat{\bell}$ with the true linear payoff $\bxi$ and the results are sensibly the same. Moreover, increasing the number of samples for computing the expected signature did not give large absolute differences. This leads us to believe that most of the difference between Linear REG and Linear BCS is due to the (non-linear) optimization having trouble finding the global minima. There are ways that could improve the optimization step in Linear REG, however it would not get any better than Linear BCS. Also, improving the truncation order will give better results at the cost of non trivial RAM and compute, as showcased in Appendix~\ref{app:arribas-timings}.
    
    \begin{figure}[H]
        \centering
        \includegraphics[width=\twoplotswidth]{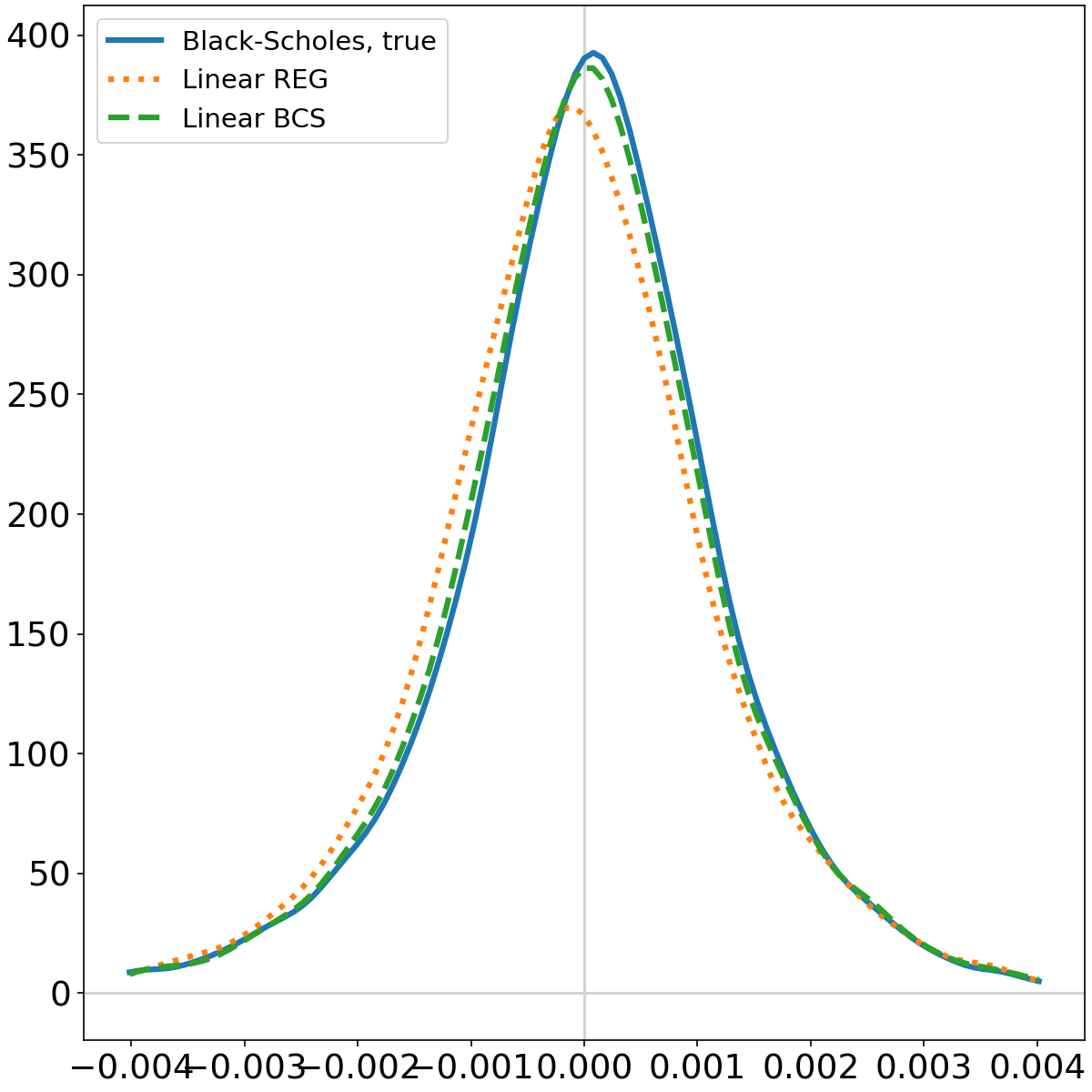}
        \caption{Polynomial payoff: Out-of-sample P\&L density of the delta hedging associated with the Black-Scholes solution (blue), Linear REG approximation (blue) and Linear BCS approximation (orange).}
        \label{fig:polyarribas-blackscholes}
    \end{figure}

    \begin{table}[H]
        \centering
        \begin{tabular}{|l|c|c|c|}
            \hline
                              & \multirow{2}{*}{Black-Scholes}      & \multicolumn{2}{c|}{Linear}     \\
            \cline{3-4}
                              &                       & REG                   & BCS                   \\
            \hline
            Mean squared P\&L & $1.78 \cdot 10^{-6}$  & $2.19 \cdot 10^{-6}$  & $1.81 \cdot 10^{-6}$  \\
            Mean P\&L         & $9.73 \cdot 10^{-6}$ & $-1.27 \cdot 10^{-4}$ & $-1.63 \cdot 10^{-5}$ \\
            \hline
        \end{tabular}
        \caption{Polynomial payoff: Out-of-sample mean and mean squared P\&L of the delta hedging associated with the Black-Scholes solution (blue) and Linear REG approximation (orange).}
        \label{tab:polyarribas-blackschole}
    \end{table}

\subsection{European call under Heston} \label{subsec:arribas-euro}

    This second experiment is done in a more realistic setting under Heston model with a European at-the-money call option with maturity 6 months. Recall the CIR dynamics of the volatility
    \begin{align} \label{eq:CIR_arribas}
        \d V_t = \kappa (\theta - V_t) \, \d t + \eta \sqrt{V_t} \, \d W_t,
    \end{align}
    where in this experiment $V_0 = \theta = 0.0625, \kappa = 2, \eta = 0.5$.
    We compare mainly two methods \textit{Linear REG} and \textit{Fourier REG}, discussed in Sections~\ref{subsec:arribas} and \ref{subsec:sighedging} respectively, with their respective `best-case-scenario' counterparts \textit{Linear BCS} and \textit{Fourier BCS}. The true Fourier pricing with Heston's characteristic function was also displayed to showcase the validity of the signature volatility method. In this framework the signature volatility is computed as the truncated algebraic solution to the square root CIR process, i.e.
    \begin{align*}
        \Sigma_t := \bracketsig{\bsigma}, \quad \bsigma = x \sum_{m=0}^\infty \sum_{\word{v} \in V_m} f_1(\word{v}) \word{v}
    \end{align*}
    where $f$ is defined with recurrence
    \begin{align*}
        f_n(\emptyword) &
        = 1
        \\
        f_n(\word{v1}) &
        = \frac{1}{2} \left( \kappa \theta - \frac{\eta^2}{4} \right) \frac{n}{x^2} f_{n-2}(\word{v}) - \frac{\kappa}{2} n f_n(\word{v})
        \\
        f_n(\word{v2}) &
        = \frac{\eta}{2} \frac{n}{x} f_{n-1}(\word{v}).
    \end{align*}
    
    In this setting, as could be expected, the truncated linear signature trading strategy lags behind the signature volatility methods as it would require an unrealistically high truncation order to approximate the call option well enough. Moreover, Fourier REG allows for a relatively good approximation of the optimal hedging ratio in terms of mean squared P\&L, but is five times more biased than Fourier BCS.
    
    \begin{figure}[H]
        \centering
        \includegraphics[width=\twoplotswidth]{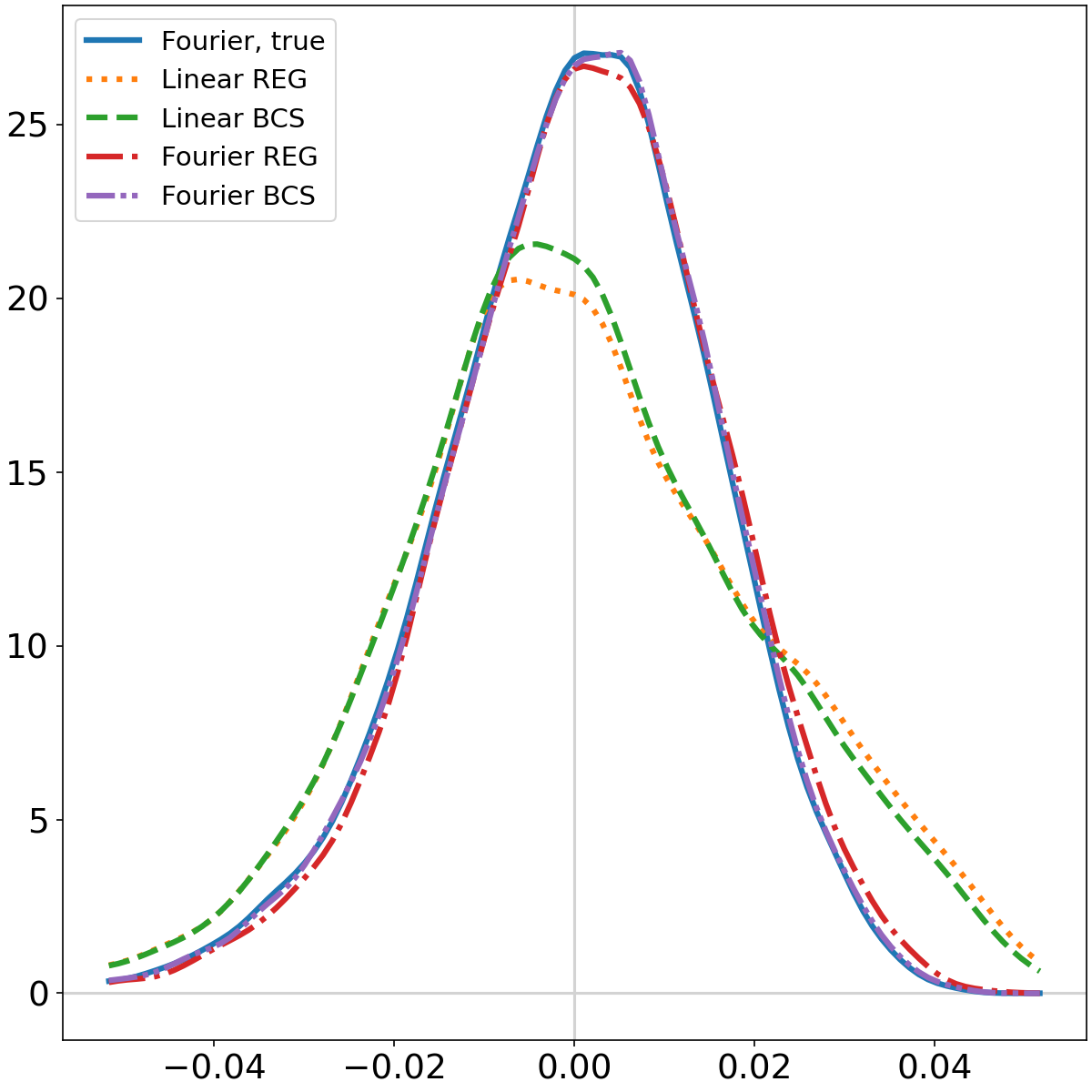}
        \caption{\eurocall: Out-of-sample P\&L density of the delta hedging associated with the Fourier solution (blue), Linear REG approximation (blue), Linear BCS approximation (orange), Fourier REG approximation (red) and Fourier BCS approximation (purple).}
        \label{fig:euroarribas-heston}
    \end{figure}

    \begin{table}[H]
        \centering
        \begin{tabular}{|l|c|c|c|c|c|}
            \hline
                              & \multirow{2}{*}{Fourier Heston}         & \multicolumn{2}{c|}{Linear}          & \multicolumn{2}{c|}{Fourier}         \\
            \cline{3-6}
                              &                       & REG                   & BCS                   & REG                   & BCS                   \\
            \hline
            Mean squared P\&L & $2.33 \cdot 10^{-4}$  & $4.66 \cdot 10^{-4}$  & $4.09 \cdot 10^{-4}$  & $2.36 \cdot 10^{-4}$  & $2.33 \cdot 10^{-4}$  \\
            Mean P\&L         & $-5.84 \cdot 10^{-5}$ & $6.78 \cdot 10^{-4}$  & $1.04 \cdot 10^{-4}$  & $-1.05 \cdot 10^{-3}$ & $2.14 \cdot 10^{-4}$  \\
            \hline
        \end{tabular}
        \caption{\eurocall: Out-of-sample mean and mean squared P\&L of the delta hedging associated with the Black-Scholes solution (blue), Linear REG approximation (orange), Linear BCS approximation (green), Fourier REG approximation (red) and Fourier BCS approximation (purple).}
        \label{tab:euroarribas-heston}
    \end{table}

    The poor performances of Linear REG and even Linear BCS highlights the non-linearity of the hedging ratio in terms of the signature of the spot process. On the contrary, the very close mean squared P\&L of Fourier REG and Fourier Heston (and Fourier BCS) suggests that under this type of stochastic volatility model \eqref{eq:dynaprice} one can readily calibrate and price options simply using observations, or estimation, of the spot volatility process.

\subsection{Asian call under delayed equation volatility} \label{subsec:arribas-asian}

    In this final experiment, we compare both methods under a non-Markovian volatility model with a path-dependent at-the-money Asian call option. For our non-Markovian model we have chosen the delayed-equation (DE) stochastic volatility model for its explicit signature linear representation. The DE process has the following dynamics
    \begin{align} \label{eq:DE_arribas}
        \d \Sigma_t &
        = \kappa ( \theta - \Sigma_t) \d t + \left( \eta + \alpha \Sigma_t + \varsigma \int_0^t e^{-\lambda (t - s)} \Sigma_s \d s \right) \d W_t,
    \end{align}
    where $\Sigma_0=\theta=0.25, \kappa=2, \eta=-0.12, \alpha=1.2, \varsigma=1, \lambda=-3$. The DE volatility is obviously mean-reverting and has its exponential moving average impacts its volatility. In this framework the signature volatility is computed as the extended mGBM
    \begin{align}
        \Sigma_t = \bracketsig{\bsigma}, \quad \bsigma = \left( \Sigma_0 \emptyword + \left( \kappa \theta - \frac{\alpha \eta}{2} \right) \word{1} + \eta \word{2} \right) \left(\emptyword + \left( \kappa + \frac{\alpha^2}{2} \right) \word{1} - \alpha \word{2} - \bm{q} \right)^{-1},
    \end{align}
    where $(\emptyword - \bell)^{-1} := \sum_{n=0}^\infty \bell \conpow{n}$ and $\bm{q}$ represents the exponential moving average dynamics
    \begin{align}
        \bm{q} &
        := \varsigma \word{1} \shuexp{\lambda \word{1}}  \left(- \frac{\alpha}{2} \word{1} + \word{2} \right).
    \end{align}

    In this framework, the gap between the \textit{model-free} linear and the \textit{volatility-free} Fourier methods deepens. Indeed, Linear REG and even Linear BCS have a mean squared P\&L about four times higher, highlighting the non-linearity of the hedging solution.
    
    \begin{figure}[H]
        \centering
        \includegraphics[width=\twoplotswidth]{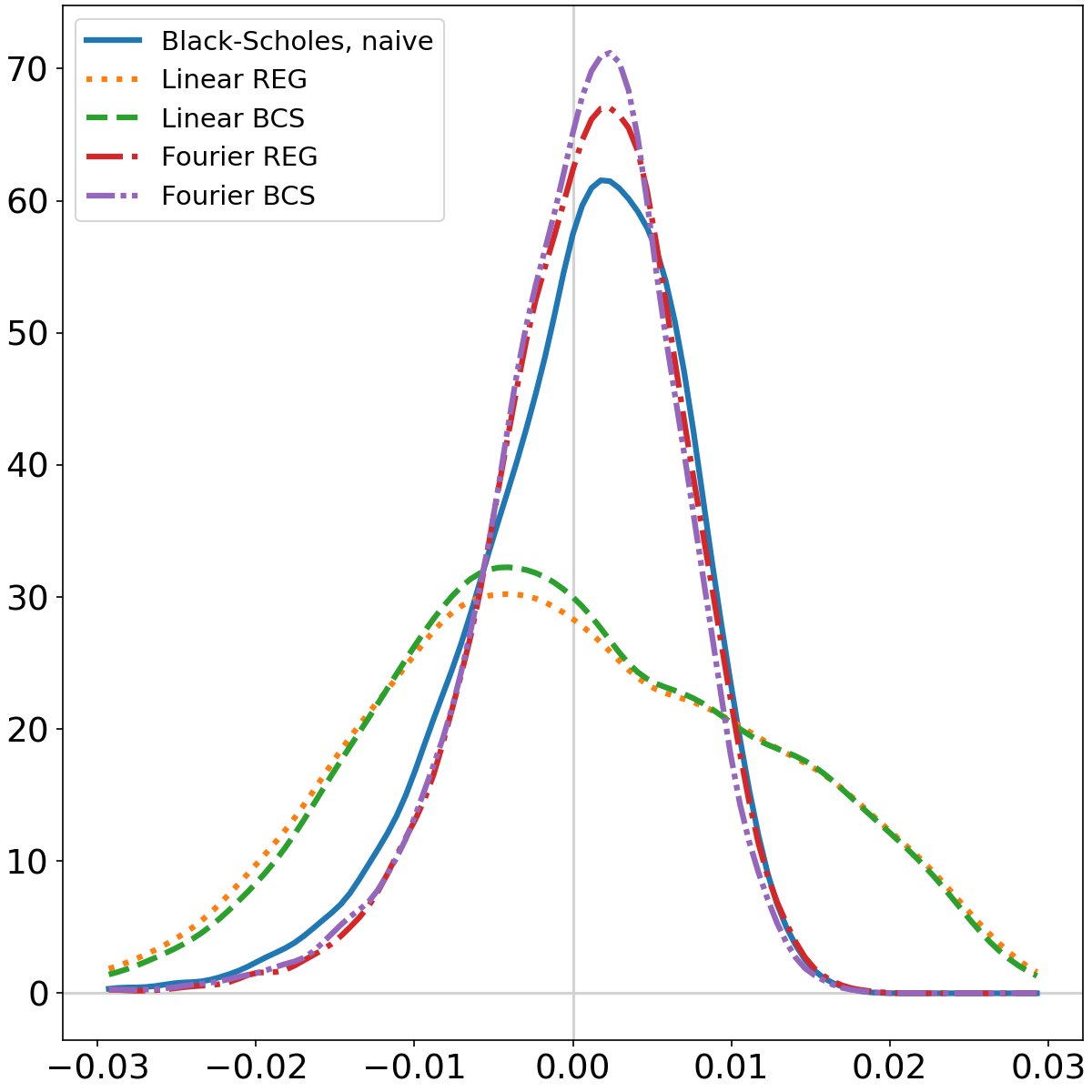}
        \caption{\asiancall: Out-of-sample P\&L density of the delta hedging associated with the Black-Scholes solution (blue), Linear REG approximation (orange), Linear BCS approximation (green), Fourier REG approximation (red) and Fourier BCS approximation (purple).}
        \label{fig:asianarribas-delayed}
    \end{figure}

    \begin{table}[H]
        \centering
        \begin{tabular}{|l|c|c|c|c|c|}
            \hline
                              & \multirow{2}{*}{Naive Black-Scholes}   & \multicolumn{2}{c|}{Linear}          & \multicolumn{2}{c|}{Fourier}         \\
            \cline{3-6}
                              &                       & REG                   & BCS                   & REG                   & BCS                   \\
            \hline
            Mean squared P\&L & $4.90 \cdot 10^{-5}$   & $1.74 \cdot 10^{-4}$  & $1.63 \cdot 10^{-4}$  & $4.00 \cdot 10^{-5}$  & $3.97 \cdot 10^{-5}$  \\
            Mean P\&L         & $1.91 \cdot 10^{-4}$   & $-1.16 \cdot 10^{-3}$  & $2.01 \cdot 10^{-5}$  & $6.82 \cdot 10^{-4}$  & $3.02 \cdot 10^{-4}$ \\
            \hline
        \end{tabular}
        
        \caption{\asiancall: Out-of-sample mean and mean squared P\&L of the delta hedging associated with the Black-Scholes solution (blue), Linear REG approximation (orange), Linear BCS approximation (green), Fourier REG approximation (red) and Fourier BCS approximation (purple).}
        \label{tab:asianarribas-delayed}
    \end{table}

    In Figure~\ref{fig:asianarribas-delayed} and Table~\ref{tab:asianarribas-delayed} we can clearly make the same conclusion as for the European call under Heston model, only stronger as the linear strategies lags behind even more than the Fourier strategies.

\appendix

\section{Algorithms} \label{app:algorithms}

    In this section we explicit the algorithms that have been used in the previous sections. More specifically, Algorithm~\ref{alg:reinforcement} describes the training of the neural networks in Section~\ref{sec:sig_nn} and, up to the initialization of the network, Linear BCS in Section~\ref{sec:sig_reg}. Algorithm~\ref{alg:sigregression} then describes the algorithm for the (non-linear) regression of the signature volatility of Fourier REG in Section~\ref{sec:sig_reg}.
    
    \begin{algorithm}[H]
        \caption{Deep-hedging training algorithm}
        
        \textbf{Data initialization}:
        \begin{algorithmic}[10]
            \State Fix $J > 0$ the number of points in the discretization of $[0, T]$,
            \State Fix $I > 0$ the number of simulations in the training set,
            \State Fix a model for the spot price and volatility,
            \State Fix a payoff function,
        \end{algorithmic}
        
        \textbf{Net initialization}:
        \begin{algorithmic}[10]
            \State Fix $Q$ the width of the feature set and $H, L > 0$ the width and depth of the hidden layers respectively,
            \State Choose the hidden unit as either vanilla or LSTM,
            \State Randomly initialize the set of trainable parameters $\theta$, i.e.~all weights and biases and $p_0$,
            \State Fix a batch size, learning rate scheduler, optimizer and regularization rate.
        \end{algorithmic}
        
        \textbf{Offline}:
        \begin{algorithmic}[10]
            \State Simulate $M$ trajectories of $S$ and $\Sigma$ given the model with $N$ time steps and compute $\Delta S_{j+1} := S_{j+1} - S_j$ for $0 \leq j < J$,
            \State Construct $(X_{t_j})_{0 \leq j \leq J-1} = (t_j, S_{t_j}, \Sigma_{t_j})_{0 \leq j \leq J-1}$ and $(\mathbb{X}_{t_j})_{0 \leq j \leq J-1}$ in the case of the SNN
            \State Compute the payoff $\xi$ for those simulations.
        \end{algorithmic}

        \textbf{Online}:
        \begin{algorithmic}[1]
            \While {step in range of maximum iteration}
                \For {batch $B$ in training set $I$}
                    \State $p_0^\theta, (\alpha_j^\theta)_{0 < j < J-1} \leftarrow f^\theta((X_j)_{0 < j < J-1})$,
                    \State $P \leftarrow p_0^\theta + \frac{1}{|B|} \sum_{i \in B} \left( \sum_{j=0}^{J-1} \alpha_j^\theta \Delta S_{j+1}^i - \xi^i \right)$,
                    \State Compute/retrieve gradients of $P$ with respect to $\theta$,
                    \State Update $\theta$ with optimizer,
                    \State Increment step,
                \EndFor
            \EndWhile
        \end{algorithmic}
        \label{alg:reinforcement}
    \end{algorithm}

    \begin{algorithm}[H]
        \caption{Regression of volatility dynamics}
        
        \textbf{Initialization}:
        \begin{algorithmic}[10]
            \State Fix a truncation orders $M$ for the linear volatility $\hat{\bsigma}$, $\widehat{M}$ for the expected Brownian signature and $\bar{M}$ for moment of $Y$ to match,
            \State Initialize $\hat{\bsigma}_0 \in \tTA[2]{M}$, either randomly or following a truncated representation,
            \State Fix $I > 0$ the number of simulations or observations and $J > 0$ the number of time steps $t_j$ of $[0, T]$,
            \State Simulate or observe $(\Sigma_{t_j}^i)_{0 \leq j \leq J, 1 \leq i \leq I}$, compute the signature of its time-augmented variant $\widehat{\mathbb{Y}}$ up to order $\bar{M}$ and finally its average over simulations $\bar{\mathbb{Y}}_{t_j} = \frac{1}{I} \sum_{i=1}^I \widehat{\mathbb{Y}}_{t_j}^i$ for $1 \leq j \leq J$,
            \State Compute $(\sigE[t_j])_{1 \leq j \leq J}$ up to order $\widehat{M}$ using Fawcett's formulae, i.e.~$\sigE = e^{\otimes (\word{1} + \frac{1}{2} \word{22}) t}$,
            \State Compute the truncated half shuffle product operator $\widetilde{\succ}: \tTA[2]{\widehat{M}} \times \tTA[2]{M} \to \tTA[2]{\widehat{M}}$ where the half shuffle product $\succ$ is defined in Definition~\ref{def:halfshuffleprod}.
        \end{algorithmic}

        \textbf{Optimization}:
        \begin{algorithmic}[1]
            \While {step $i$ in range of maximum iteration $I$}
                \State $\bell^\emptyword \gets \emptyword$
                \For {word $\word{v}$ in $\cup_{m=0}^{\bar{M}} V_m$, recursively}
                    \State $\bell^\word{v1} \gets \bell^\word{v} \word{1}$,
                    \State $\bell^\word{v2} \gets \bell^\word{v} \widetilde{\succ} \hat{\bsigma}_i$,
                \EndFor
                \State Compute loss $\mathcal{L} \gets \sum_{j=1}^J \sum_{m=1}^{\bar{M}} \sum_{\word{v} \in V_m} (\langle \bell^\word{v}, \sigE[t_j] \rangle - \bar{\mathbb{Y}}_{t_j}^\word{v})^2$,
                \State Update $\hat{\bsigma}_i$ from optimizer,
            \EndWhile
            \State \Return $\hat{\bsigma} := \hat{\bsigma}_I$.
        \end{algorithmic}
        \label{alg:sigregression}
    \end{algorithm}
    Note that a regularization can be easily included in the loss. Furthermore, one could easily enforce a structure over $\hat{\bsigma}$, e.g.~the equivalent of a linear SDE $\hat{\bsigma} = \left( \bsigma_0 \emptyword + \hat{\bm{p}} \right) \left(\emptyword - \hat{\bm{q}} \right)^{-1}$. \\

    \begin{sqremark}
        It is important to note that a random initialization for $\hat{\bsigma}_0$ will lead to very variable results because of the high non-convexity of the problem. A solution to alleviate this problem was to initialize as a reasonable candidate, i.e.~estimate the parameters under some model that admits a linear representation and take this truncated representation as initial guess $\hat{\bsigma}_0$. In our case we used the Ornstein-Uhlenbeck representation in \cite[Proposition 3.1]{sigvolfourier} after computing the maximum likelihood estimators $\hat{\kappa}, \hat{\theta}$ and $\hat{\eta}$. We refer to \cite{maxlikelihood_oucir} for their computations.
    \end{sqremark}

\section{Log-signature or not?} \label{app:sig_order_sup2}

    In this section we are interested in two main questions of Section~\ref{sec:sig_nn}. The first one, which is the easiest to answer, asks whether one should use the log-signature, defined below, instead of the signature to reduce the number of terms. Indeed, the log-signature introduced in \cite{chen_logsig} preserves the full information of the signature in a more compact form, as the formal logarithm is bijective, see \cite{lyons_drivenroughpaths}. To give an intuition, the log-signature removes polynomial redundancies present in the signature, e.g.~all terms up to order $M$ of the signature can be computed with polynomials of degree at most $M$ of terms of log-signature of order up to $M$. \\
    
    The short answer to this first question is that the use of the log-signature relatively under-performs and takes longer to train. We speculate that it is due to the additional requirement of the network to learn polynomials, e.g.~when $S_t^2$ was perfectly accessible with the signature, the network has the approximate the square when trained on the log-signature.

    \begin{definition}[Formal logarithm]
        Let $\bp \in \eTA$ such that $\bp^\emptyword = 0$, then the logarithm map of $\bp$ is defined
        $$ \log(\emptyword + \bp) = \sum_{n \geq 1} \frac{(-1)^{n-1}}{n} \bp^{\otimes n}. $$
    \end{definition}

    \begin{definition}[Log-signature] \label{def:logsig}
        The log-signature of a path $X$ denoted by $\log \sigX$ is the logarithm of the signature of the path $X$.
    \end{definition}
    The reader might also refer to \cite{iisignature} for more details on how to compute it.
    
    \begin{sqexample}
        Let $X_t = (t, S_t, \Sigma_t)$, then the first few orders of $\sigX$ are given by
        \begin{equation}
            \log \sigX^0 = 0,
            \quad
            \log \sigX^1 =
            \begin{pmatrix}
                t \\
                S_t \\
                \Sigma_t
            \end{pmatrix},
            \quad
            \log \mathbb{X}^2 = \frac{1}{2}
            \begin{pmatrix}
                \int_0^t s \d S_s - \int_0^t S_s \d s
                \\
                \int_0^t s \d \Sigma_s - \int_0^t \Sigma_s \d s
                \\
                \int_0^t S_s \circ \d \Sigma_s - \int_0^t S_s \circ \d \Sigma_s
            \end{pmatrix}.
        \end{equation}
    \end{sqexample}

    The second question we study in this section, which is harder to answer, arises naturally when constructing the signature objects: what truncation order that leads to the best hedging and fastest training networks. The conclusions are much less clear but we can still see that most of the performance can be attributed to the first 2 or 3 orders of the signature. A similar observation was made by \citep*{bayer_sigcontrols} in an optimal control setting. Note that the terms of the signature of order 2 of some $X=(X^1, X^2)$ are the squares of $\frac{1}{2} X^i$ and linear combinations of the product $X^1 X^2$ and the Lévy area $\int_0^t X^1 \d X^2 - \int_0^t X^2 \d X^1$. \\

    All numerical experiments have been averaged over 50 iterations for the SNN and 10 for the RNN, lower for obvious time reasons. By iteration we mean a random initialization, a training and a test. This is to ensure that we are not making conclusions with networks that performed better or worse because of their initialization. \\

    In Figure~\ref{fig:asian_trunc} we can see that both signature and log-signature plateau together after truncation order 2, close to the Fourier solution. This seems to imply that the required information to converge perfectly to the Fourier solution is found in much higher orders of the signature. On the contrary, Figure~\ref{fig:lookback_trunc} shows that the signature and log-signature have a significantly different impact on training with the Log-Signature NN plateauing relatively high and early in the truncation, whereas the Signature NN goes lower and does not seems to plateau under the Heston model.
    
    \begin{figure}[H]
        \centering
        \subfloat[\centering \heston]{
        \includegraphics[width=\twoplotswidth]{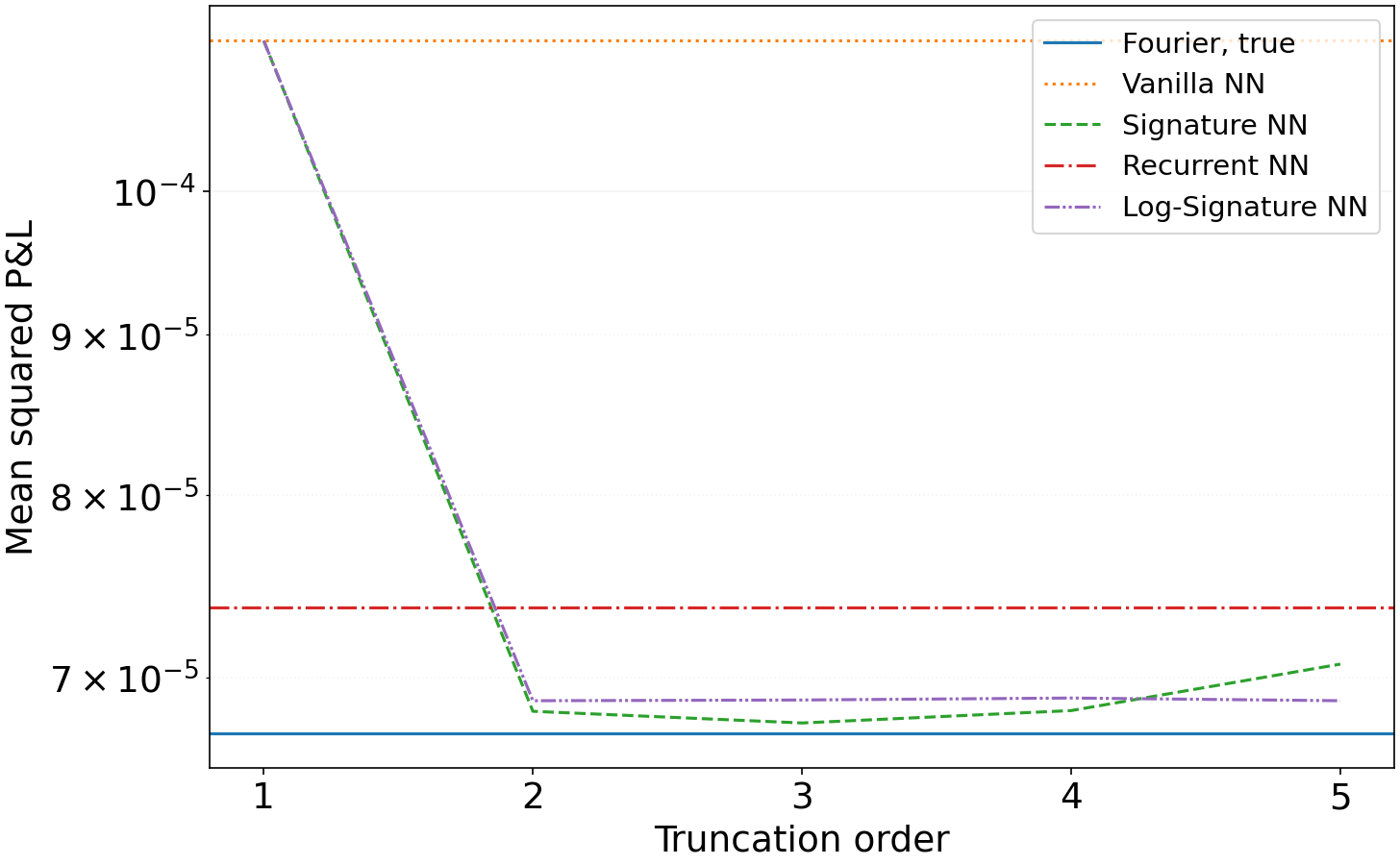}
        }
        \quad
        \subfloat[\centering \fbergomi]{
        \includegraphics[width=\twoplotswidth]{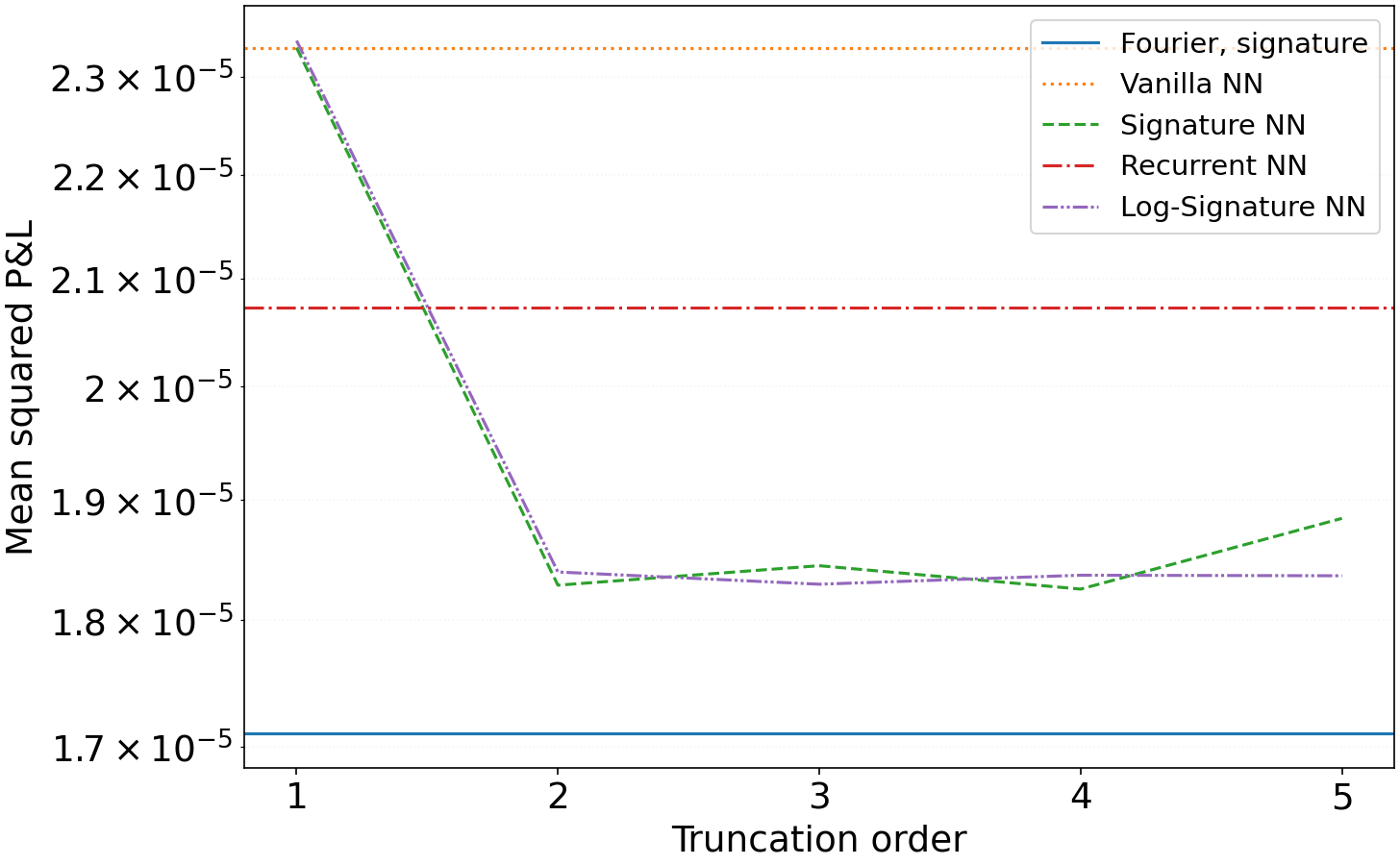}
        }
        \caption{\asiancall: \TRUNCcaption}
        \label{fig:asian_trunc}
    \end{figure}
    
    \begin{figure}[H]
        \centering
        \subfloat[\centering \heston]{
        \includegraphics[width=\twoplotswidth]{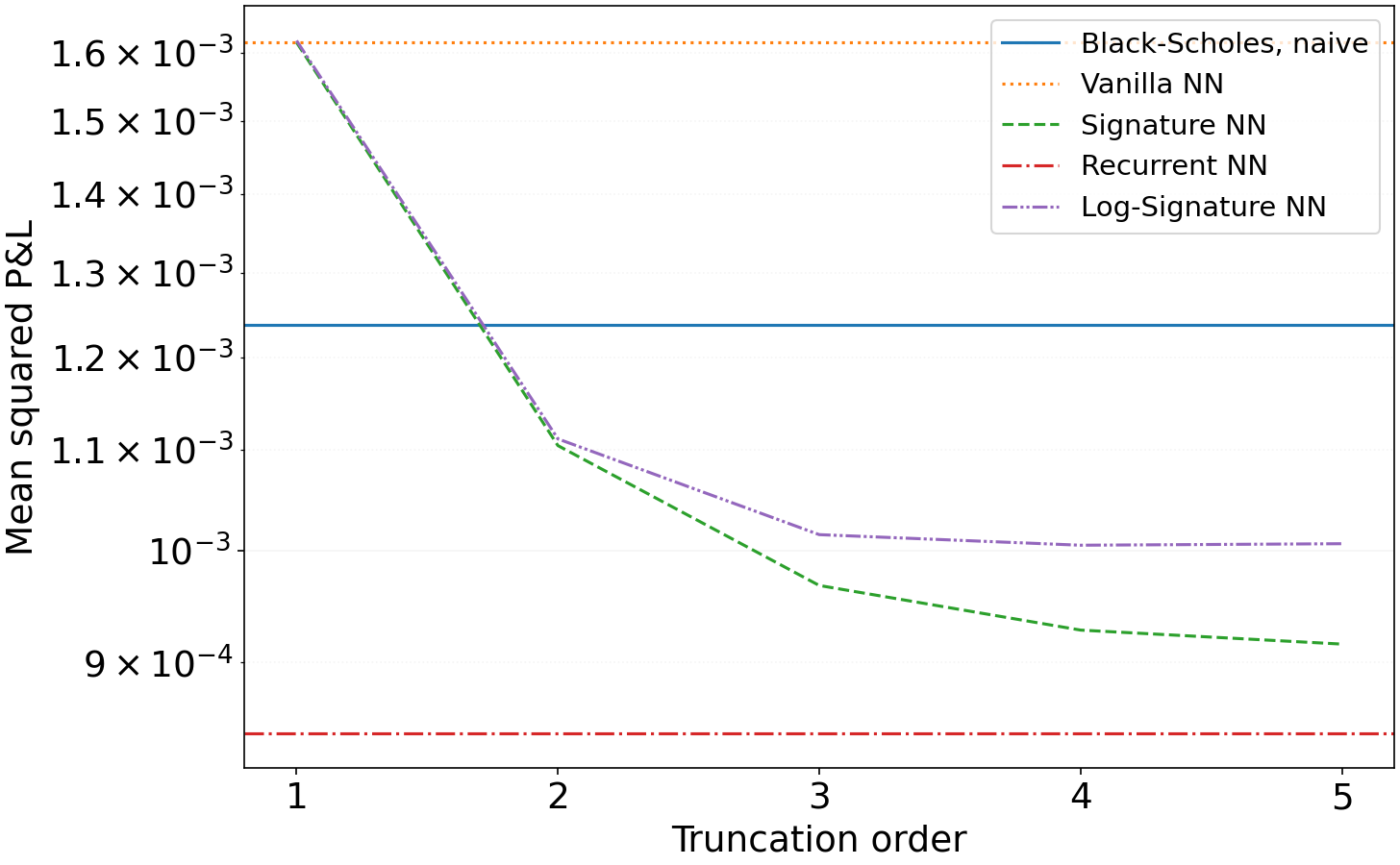}
        }
        \quad
        \subfloat[\centering \fbergomi]{
        \includegraphics[width=\twoplotswidth]{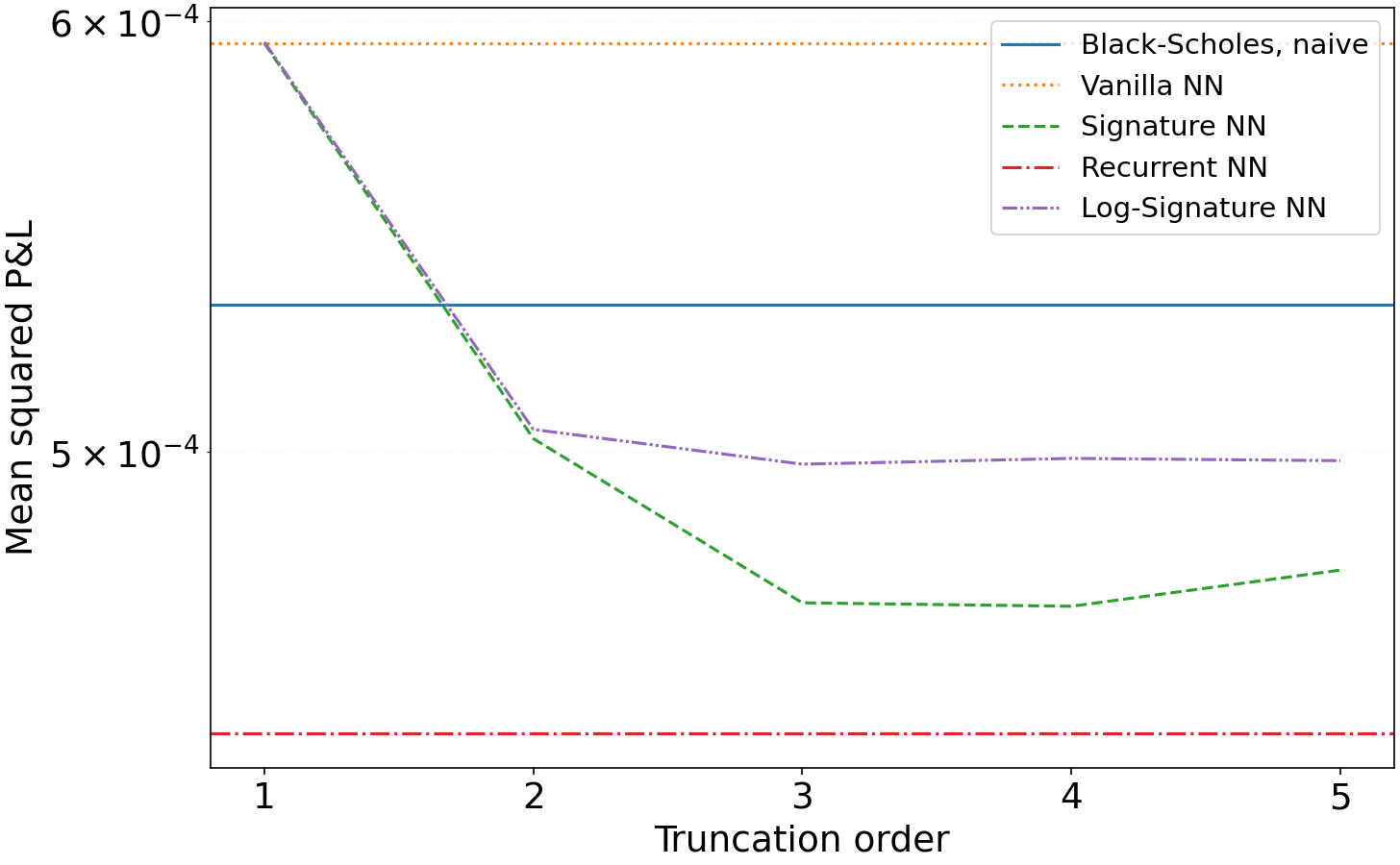}
        }
        \caption{\lookbackcall: Out-of-sample mean squared P\&L of the delta hedging associated with the the Black-Scholes naive solution (blue), Vanilla NN (orange), Signature NN (green), Recurrent NN (red) and Log-Signature NN (purple), for several truncation orders.}
        \label{fig:lookback_trunc}
    \end{figure}

    The last couple of figures show in more details the training of the Log-Signature NN, Figures~\ref{fig:asian_time_log} and \ref{fig:lookback_time_log} and the Signature NN, Figures~\ref{fig:asian_time} and \ref{fig:lookback_time}, under the different models (a) and (b) for the Asian call option, Figurs~\ref{fig:asian_time_log} and \ref{fig:asian_time}, and for the look-back call option, Figures~\ref{fig:lookback_time_log} and \ref{fig:lookback_time}. It shows in abscissa the time of training (averaged over the 50 iterations) and in ordinate the mean squared P\&L over the 20,000 OOS simulations. In dashed gray the benchmark and in color the different truncation orders of the (Log-)Signature NN. It first shows the very clear improvement between the Vanilla NN, i.e.~$M=1$, and the Signature NN, $M \geq 2$. Then it also shows that the Log-Signature NN trains slower when compared to the Signature NN for the same mean squared P\&L.
    
    However, apart from a small improvement between truncation 2 and 3 for the look-back call option, it seems that higher orders either don't help or actually delay the convergence. \\

    Some other interesting artifact is that, even averaged over 50 iterations, there are clear \textit{breakthroughs} in the learning process for the Asian option in Figures~\ref{fig:lookback_time_log} and \ref{fig:asian_time}, specifically with the log-signature.
    
    \begin{figure}[H]
        \centering
        \subfloat[\centering \heston]{
        \includegraphics[width=\twoplotswidth]{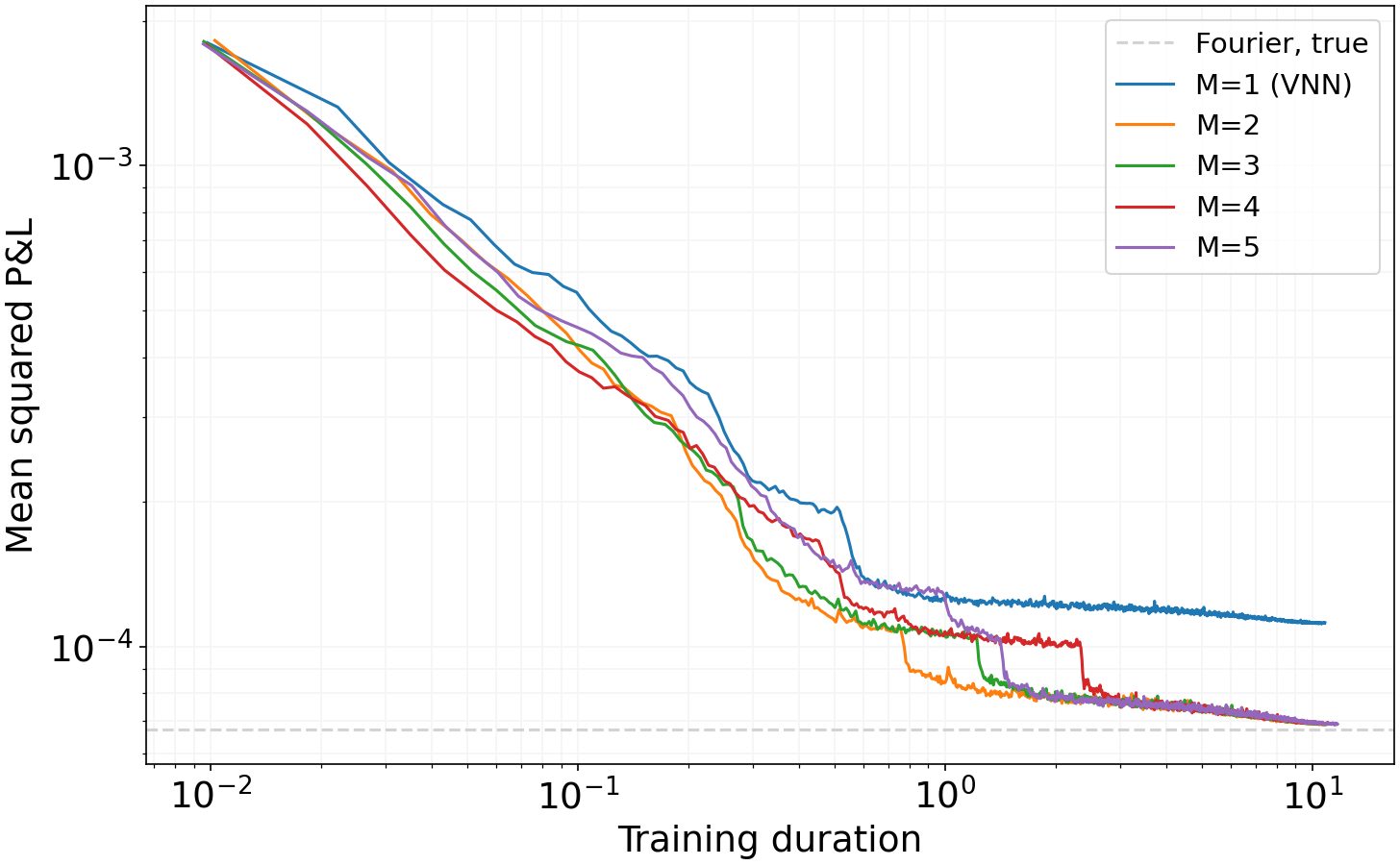}
        }
        \quad
        \subfloat[\centering \fbergomi]{
        \includegraphics[width=\twoplotswidth]{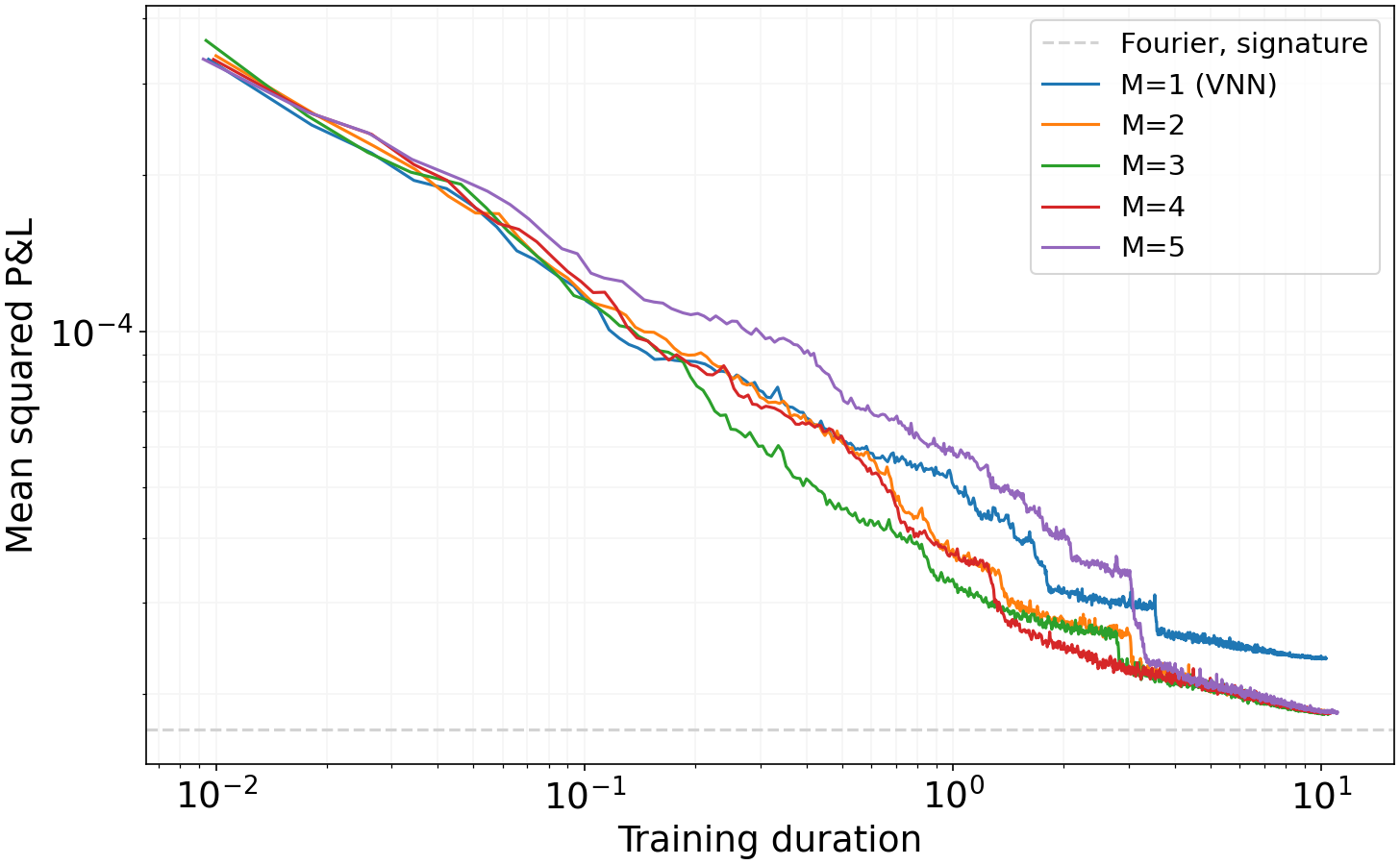}
        }
        \caption{Log-Signature, \asiancall: \CONVcaptionLog}
        \label{fig:asian_time_log}
    \end{figure}
    
    \begin{figure}[H]
        \centering
        \subfloat[\centering \heston]{
        \includegraphics[width=\twoplotswidth]{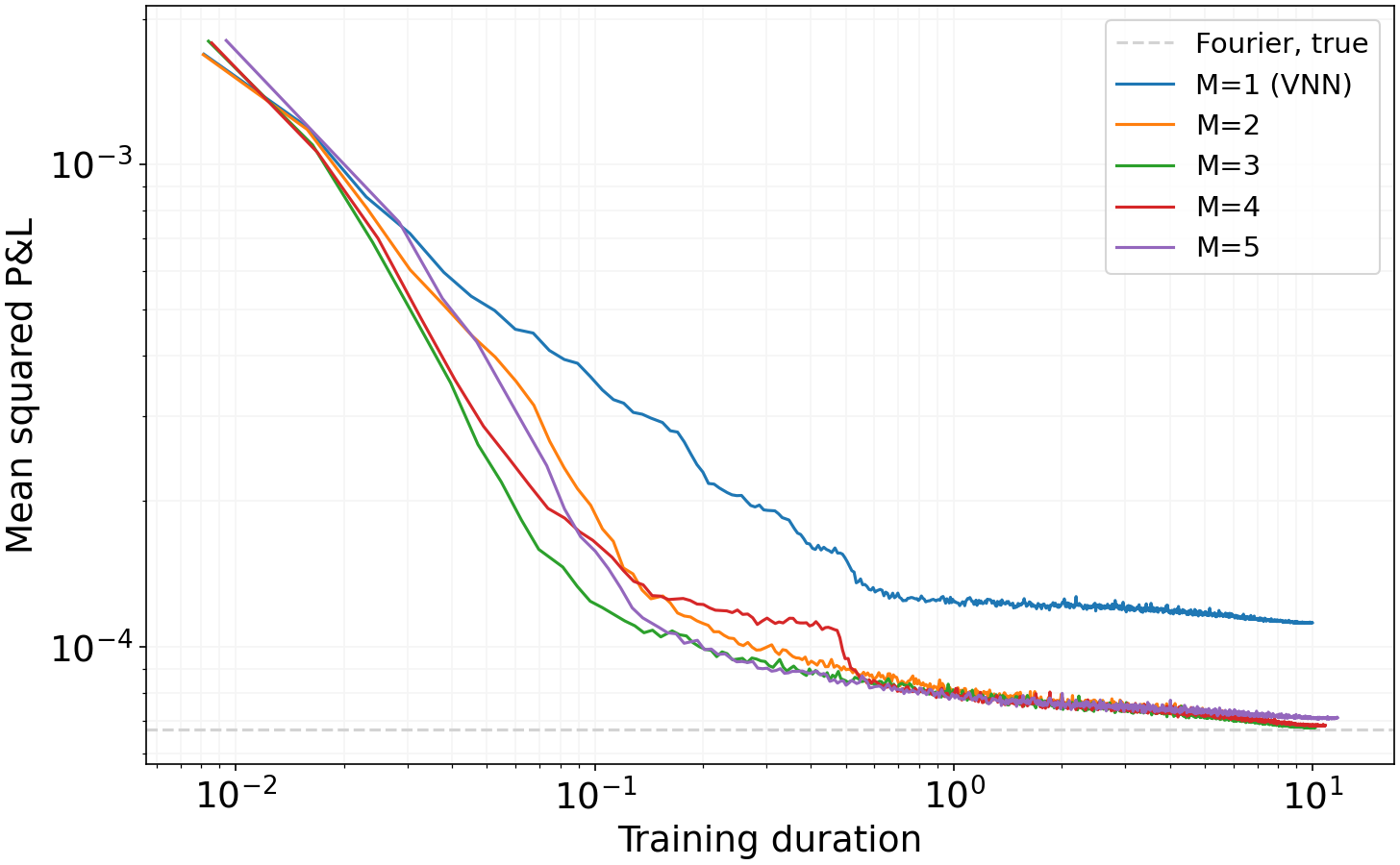}
        }
        \quad
        \subfloat[\centering \fbergomi]{
        \includegraphics[width=\twoplotswidth]{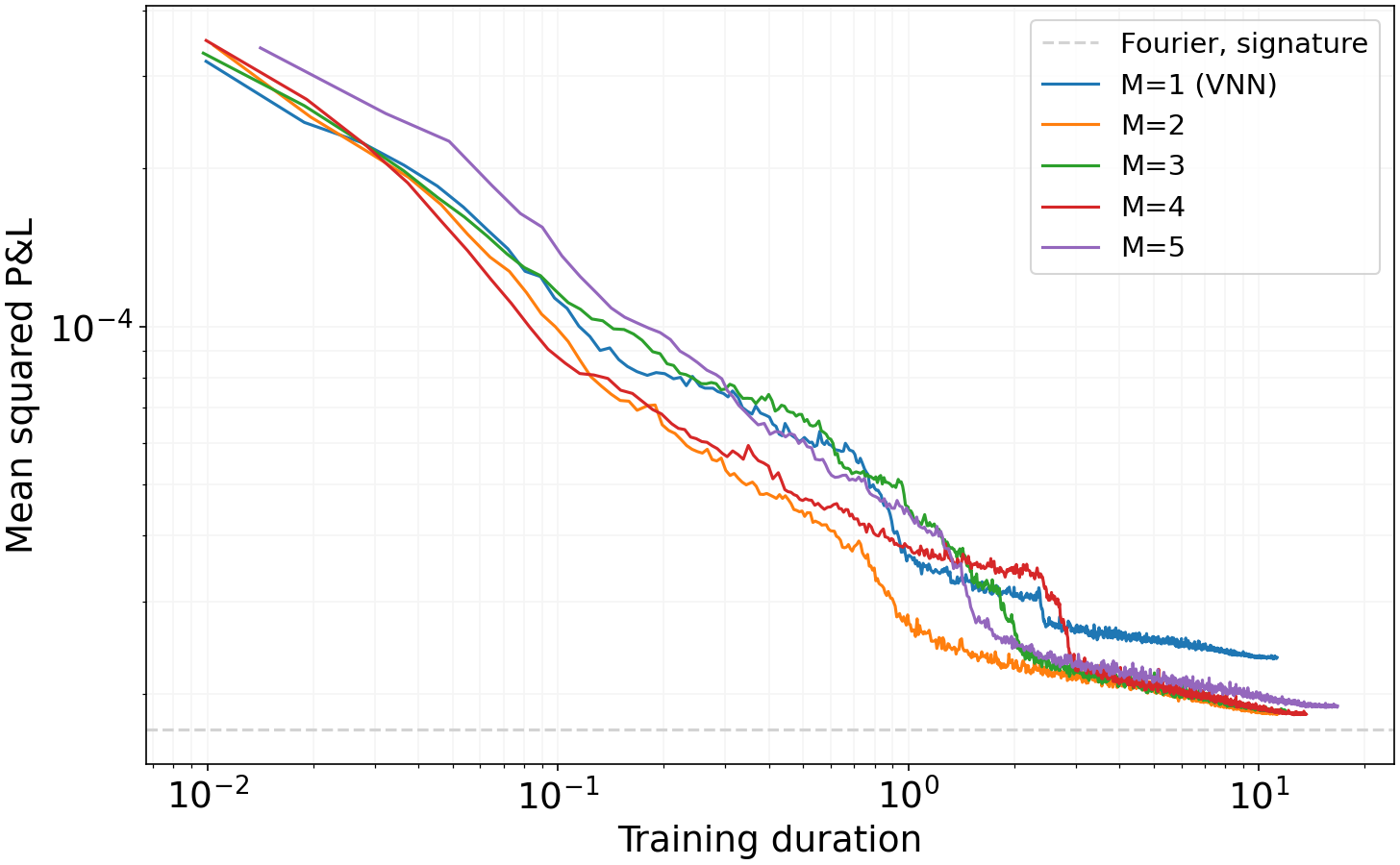}
        }
        \caption{Signature, \asiancall: \CONVcaption}
        \label{fig:asian_time}
    \end{figure}

    \begin{figure}[H]
        \centering
        \subfloat[\centering \heston]{
        \includegraphics[width=\twoplotswidth]{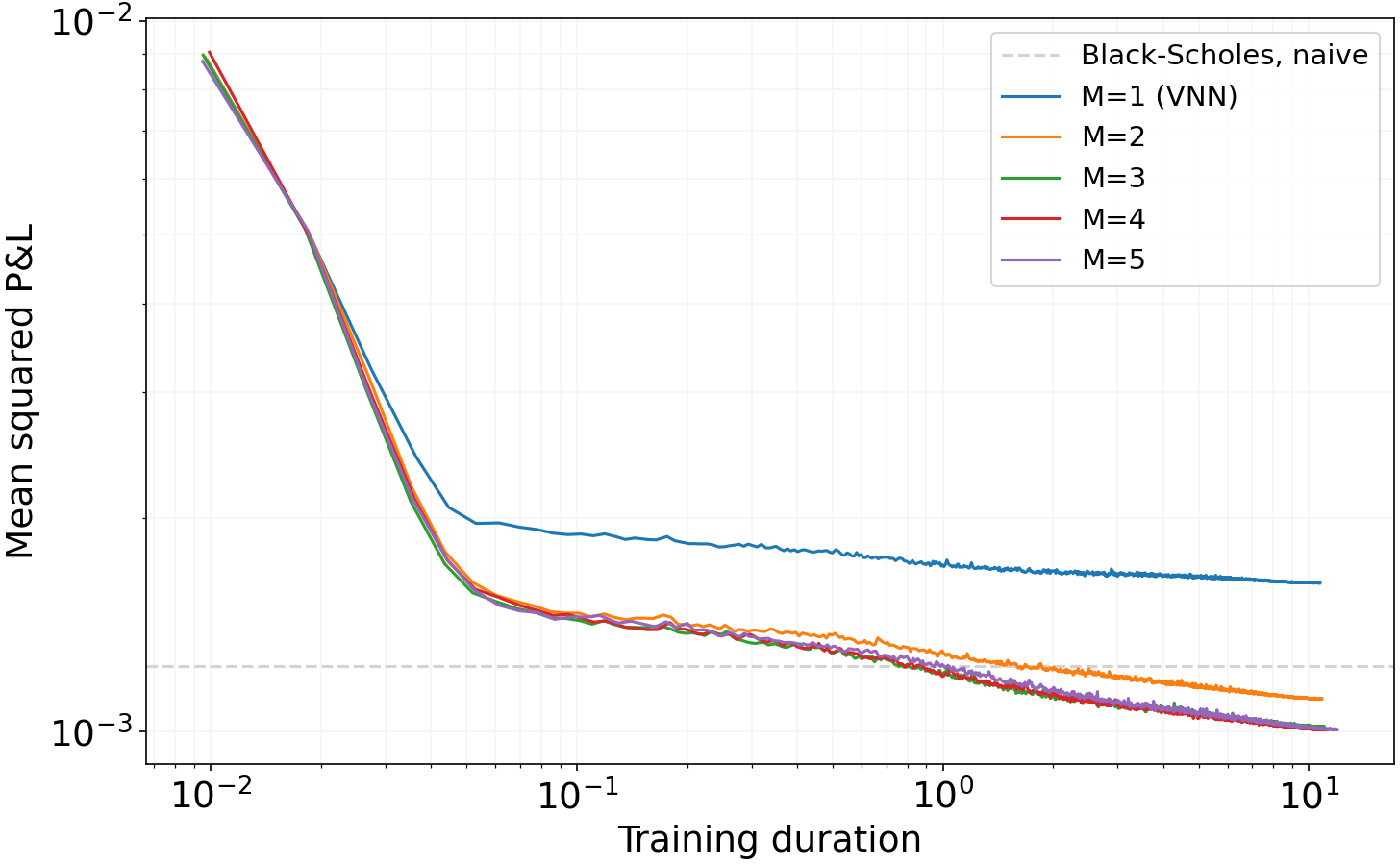}
        }
        \quad
        \subfloat[\centering \fbergomi]{
        \includegraphics[width=\twoplotswidth]{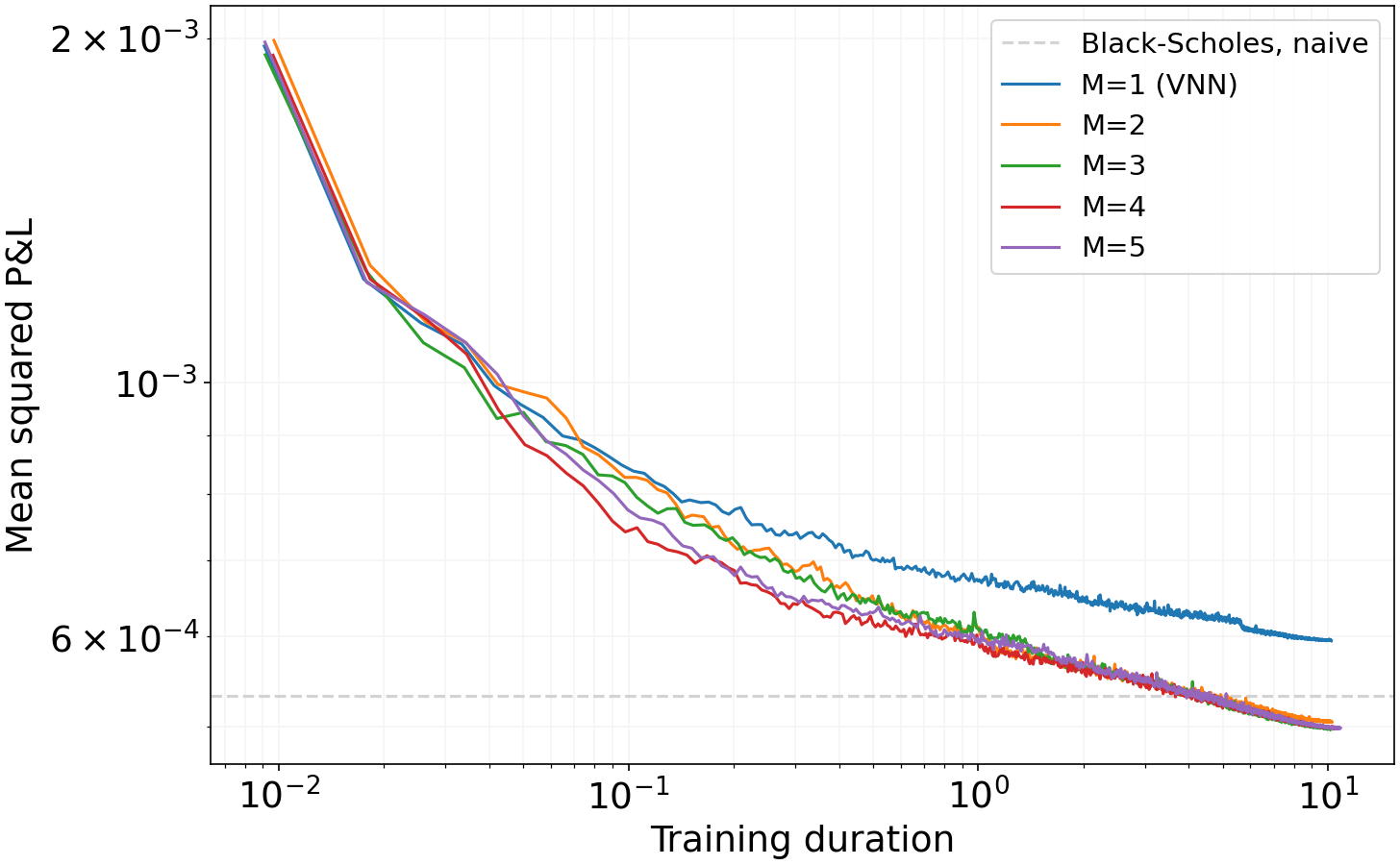}
        }
        \caption{Log-Signature, \lookbackcall: \CONVcaptionLog}
        \label{fig:lookback_time_log}
    \begin{figure}[H]
    
    \end{figure}
        \centering
        \subfloat[\centering \heston]{
        \includegraphics[width=\twoplotswidth]{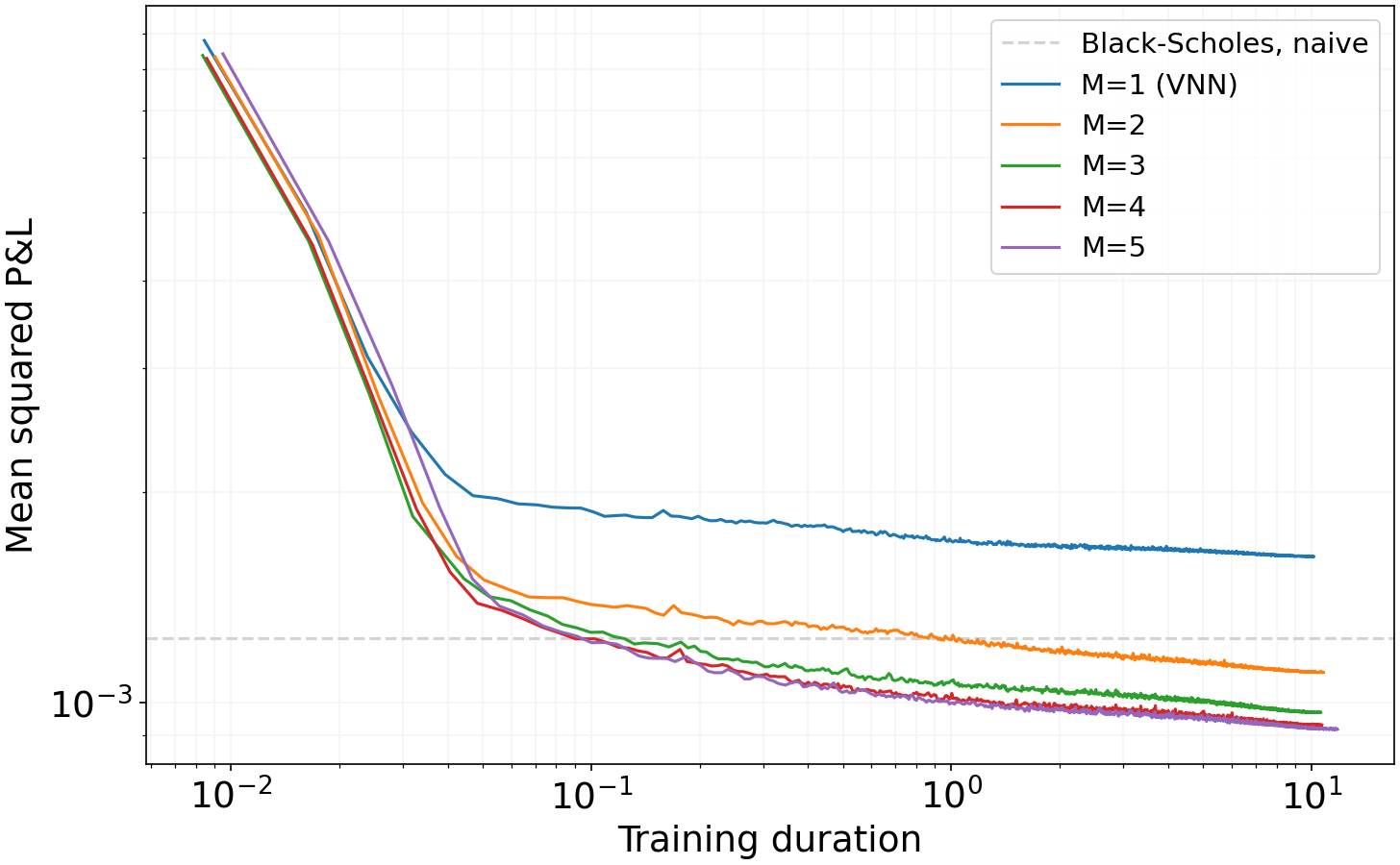}
        }
        \quad
        \subfloat[\centering \fbergomi]{
        \includegraphics[width=\twoplotswidth]{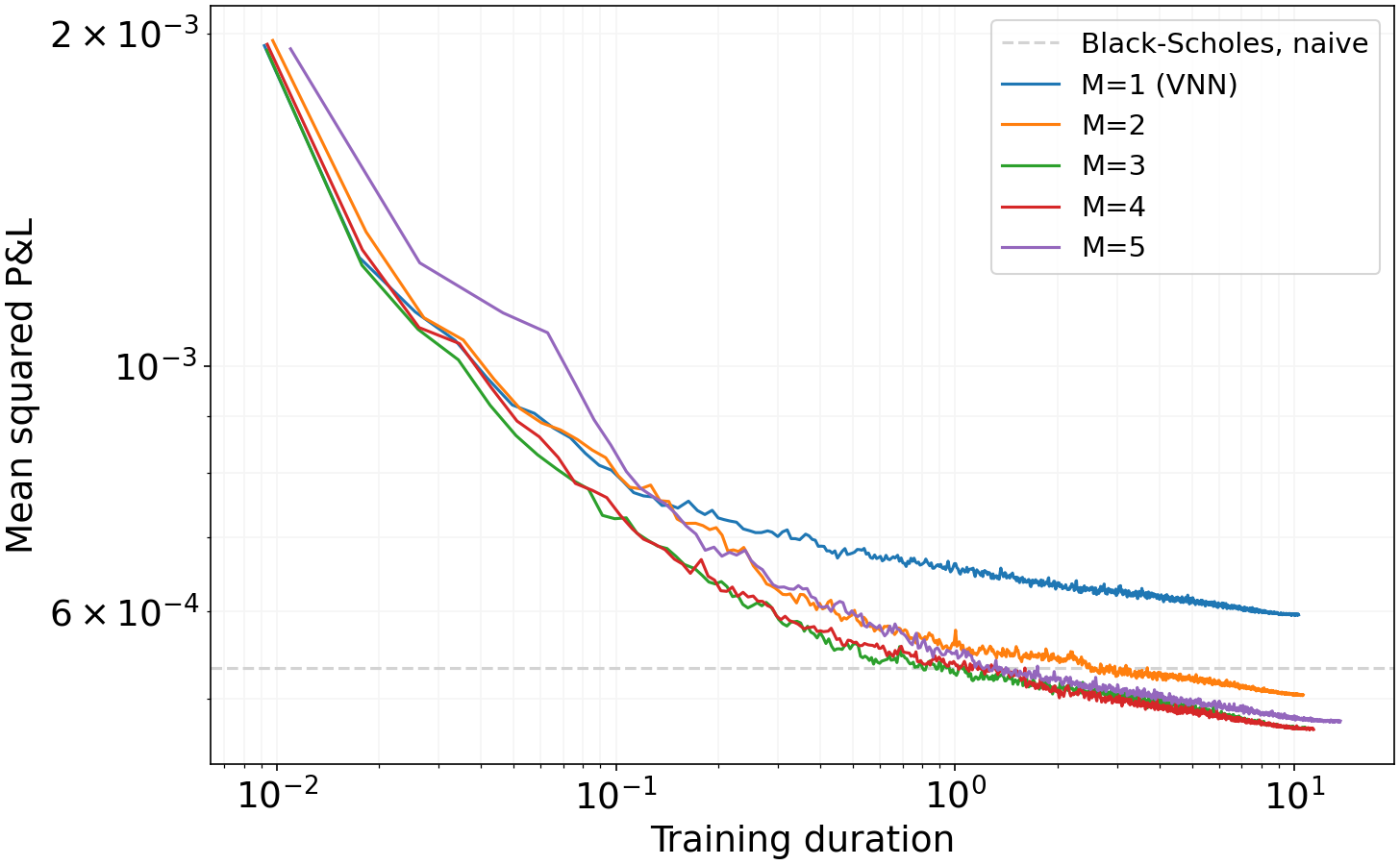}
        }
        \caption{Signature, \lookbackcall: \CONVcaption}
        \label{fig:lookback_time}
    \end{figure}
    
    Note that we did not include in the previous figures the training of the Recurrent NN as it would have been offset by an order of 30 compared to the curves of the Signature and Log-Signature NNs.

\section{Timings} \label{app:arribas-timings}

    In this section we describe briefly the timings and some RAM requirements of Section~\ref{sec:sig_reg} as they are not trivial for every truncation order. Moreover, as compute times are quite similar across examples, we only show for the case of the Heston model and European call option of Section~\ref{subsec:arribas-euro}. Note that all computations were only made on the training set of 10,000 samples of 126 time steps. Also note that the lead-lag transform holds twice as many time steps and should have twice as many dimensions even though this doubling can be averted for the first (time) dimension as it has no quadratic variation.

    \begin{table}[H]
        \centering
        \subfloat[\centering Linear REG.]{
        \begin{tabular}{|c|c|c|c|c|c|}
            \hline
            Truncation & $\mathbb{X}_T^{\LL, i} \in \tTA[3]{M}$ & $\widehat{\bxi} \in \tTA[3]{M}$ & $\bar{\shuprod}$ & $\E[\mathbb{X}_T^\LL] \in \tTA[3]{2M}$ & $\widehat{\bell} \in \tTA[2]{M-1}$ \\
            \hline
            \multirow{2}{*}{$M=2$} & 31.1 ms  & 557 µs  & 136 µs   & 297 ms   & 415 ms   \\
                                   &          &         &          &          &          \\
            \hline
            \multirow{2}{*}{$M=3$} & 57.7 ms  & 1.06 ms & 3.01 ms  & 688 ms   & 467 ms   \\
                                   &          &         &          &          &          \\
            \hline
            \multirow{2}{*}{$M=4$} & 76.8 ms  & 4.06 ms & 80.5 ms  & 4.31 sec & 539 ms   \\
                                   &          &         &          &          &          \\
            \hline
            \multirow{2}{*}{$M=5$} & 606 ms   & 17.7 ms & 2.45 sec  & 32.5 sec & 4.27 sec \\
                                   & 27.7 MiB &         & 243 MiB  &          &          \\
            \hline
            \multirow{2}{*}{$M=6$} & 606 ms   & 88.8 ms & 87.3 sec  & 355 sec  & 102 sec  \\
                                   & 83.4 MiB &         & 6.33 GiB & 6.08 MiB &          \\
            \hline
        \end{tabular}
        }
        \quad
        \subfloat[\centering Fourier REG.]{
        \begin{tabular}{|c|c|c|c|c|c|}
            \hline
            Truncation & $\widetilde{\succ}$ & $\bell(\bsigma) \in \tTA[2]{M+4} \times \tTA[2]{M+2}$ & $\hat{\bsigma} \in \tTA[2]{M}$ & $\bpsi \in [0, T] \times \tTA[2]{2M}$ & $\alpha^*$ \\
            \hline
            \multirow{2}{*}{$M=2$} & 253 µs   & 8.00 ms & 1.51 sec & 1.04 ms  & 239 µs   \\
                                   &          &         &          &          &          \\
            \hline
            \multirow{2}{*}{$M=3$} & 1.05 ms  & 13.7 ms & 3.12 sec & 6.28 ms  & 292 µs   \\
                                   &          &         &          &          &          \\
            \hline
            \multirow{2}{*}{$M=4$} & 4.36 ms  & 29.5 ms & 6.47 sec & 64.4 ms  & 399 µs   \\
                                   &          &         &          &          &          \\
            \hline
            \multirow{2}{*}{$M=5$} & 16.4 ms  & 65.2 ms & 15.8 sec & 703 ms   & 755 µs   \\
                                   &          &         &          & 39.7 MiB &          \\
            \hline
            \multirow{2}{*}{$M=6$} & 58.9 ms  & 193 ms  & 75.2 sec & 8.15 sec & 2.69 ms  \\
                                   &          &         & 39.9 MiB & 158 MiB  &          \\
            \hline
        \end{tabular}
        }
        \caption{\eurocall: Average timings for the different significant steps of each method. The number of iterations over which the timings have been averaged was determined by the Python \textit{magic} \texttt{\%timeit}. The RAM used to hold the object is also displayed when it is higher than 10 MiB.}
        \label{tab:arribas_time}
    \end{table}
    In Table~\ref{tab:arribas_time}-(a), $\bar{\shuprod}: \tTA[3]{M} \times \tTA[3]{M} \to \tTA[3]{2M}$ is the truncated shuffle product stored as a $n \times 4$ integer matrix where each column represents respectively the left input coordinate, the right input coordinate, the output coordinate and the number of occurrences of this trio. Similarly, in Table~\ref{tab:arribas_time}-(b), $\widetilde{\succ}: \tTA[2]{M+2} \times \tTA[2]{M} \to \tTA[2]{M+2}$ is the truncated half shuffle product stored in the same way. This way of storing the shuffle and half shuffle operators is much more memory efficient than having a naive sparse rank 3 tensor and allows for much faster vectorized and parallelizable \textit{scatter} operations, that can run on the GPU, than other alternatives. \\
    
    Note that the RAM usage is \textbf{not} the peak RAM usage, which is usually much higher and not trivial to control, especially for the expected signature at high orders which we had to computed in batches. Also note that the time to compute the optimal Fourier strategy $\alpha^*$ in Table~\ref{tab:arribas_time}-(b) is \textit{per trajectory}, which means in the order of seconds for 10,000 samples.

\bibliographystyle{plainnat}
\bibliography{main.bib}

\end{document}